\documentclass{article}

    \PassOptionsToPackage{numbers, compress}{natbib}

\usepackage[preprint]{neurips_2026}

\usepackage[utf8]{inputenc} 
\usepackage[T1]{fontenc}    
\usepackage{hyperref}       
\usepackage{url}            
\usepackage{booktabs}       
\usepackage{amsfonts}       
\usepackage{nicefrac}       
\usepackage{microtype}      
\usepackage{xcolor}         

\usepackage{graphicx}

\usepackage{url}

\def\figdir{Figs/}

\usepackage{amsmath}
\usepackage{amssymb}
\usepackage{mathtools}
\usepackage{amsthm}

\theoremstyle{plain}
\newtheorem{theorem}{Theorem}[section]

\theoremstyle{definition}

\theoremstyle{remark}

\usepackage{algorithm}
\usepackage{algorithmic}
\usepackage{multirow}
\usepackage[labelformat=simple]{subcaption}

\usepackage[most]{tcolorbox}

\def\figdir{Figs/}

\usepackage{bm}
\newcommand{\req}[1]{Eq.~(\ref{#1})}
\newcommand{\rfig}[1]{Fig.~\ref{#1}}
\newcommand{\rtab}[1]{Tab.~\ref{#1}}
\newcommand{\x}{\bm{\mathrm{x}}}
\newcommand{\y}{\bm{\mathrm{y}}}
\DeclareMathOperator*{\argmax}{arg\,max}
\DeclareMathOperator*{\argmin}{arg\,min}
\newtheorem*{theorem*}{Theorem}

\title{Test-time reward-guided alignment of language models \\by
importance sampling on pre-logit space}

\author{%
  Sekitoshi Kanai\\
  NTT, Inc.\\
  Tokyo, Japan\\
  \texttt{sekitoshi.kanai@ntt.com} \\
  \And
  Tsukasa Yoshida\\
  NTT, Inc.; Toyohashi University of Technology\\
  Tokyo, Japan; Aichi, Japan\\
   \And
   Hiroshi Takahashi \\
  NTT, Inc.\\
  Tokyo, Japan\\
   \And
   Haru Kuroki \\
   The University of Osaka\\
   Osaka, Japan \\
   \And
   Kazumune Hashimoto \\
   The University of Osaka\\
   Osaka, Japan \\
  }

\begin{document}

\maketitle

\begin{abstract}
Test-time alignment of large language models (LLMs)
attracts attention because fine-tuning of LLMs requires high computational costs.
In this paper, we propose 
a new test-time reward-guided alignment method called
adaptive importance sampling on pre-logits (AISP) on the basis of the sampling-based model predictive control
with the stochastic control input.
AISP applies the Gaussian perturbation into pre-logits, which are outputs of the penultimate layer,
so as to maximize expected rewards with respect to the mean of the perturbation.
We demonstrate that the optimal mean is obtained by importance sampling with sampled rewards. 
AISP outperforms best-of-n sampling in terms of rewards over the number of used samples
and achieves higher rewards than other reward-based test-time alignment methods.\looseness=-1
\end{abstract}

\section{Introduction}
Alignment of large language models (LLMs) is a vital technique 
to enable the safe and widespread use of LLMs in real-world applications.
A promising alignment method is reinforcement learning from human feedback (RLHF)~\citep{ouyang2022training,christiano2017deep,ziegler2019fine,bai2022training}.
However, RLHF imposes a heavy computational burden
since fine-tuning LLMs requires high computational costs~\citep{rafailov2023direct,kong2024aligning,hu2022lora}.
To address this, test-time (also known as inference-time and decoding-time) alignment attracts attention~\citep{kong2024aligning,li2024cascade,snell2024scaling,huang2025best,li2024rain}.

Test-time alignment aligns LLMs with human preference without updating parameters of LLMs.
This paper focuses on test-time reward-guided alignment methods that
find the optimal responses in terms of maximizing the score of a given reward.
To this goal, best-of-n sampling (BoN) is a simple but effective method,
which selects the response that achieves the highest reward values 
from $N$ generated responses from the base LLMs~\citep{snell2024scaling,lightman2023let,brown2024large,sessa2025bond}.
Though BoN can asymptotically optimize the same objective function as KL-constrained reinforcement learning (RL)~\citep{yang2024asymptotics},
there might be room for improvements, such as in sample efficiency, because it does not actively explore the optimal responses.
As another line of research, \citet{kong2024aligning} formalized test-time alignment as the optimal control problem and proposed RE-Control inspired by control theory.
RE-Control applies an external control signal to the representations of LLMs and optimizes these input trajectories.
Though RE-Control can actively explore the optimal responses by the control input,
it needs to train a value function using a reward model: i.e., it requires computation and storage costs for training including dataset collection.
\textit{Can LLMs be controlled by the training-free methods to explore the optimal response?}

In this paper, we propose a new test-time alignment inspired by sampling-based control methods without a training process.
Traditional optimal control theory can optimize input trajectories without any training process by solving differential equations such as the Pontryagin’s maximum principle. 
However, these methods are not applicable to LLM alignment because LLMs are nonlinear, complicated and large-scale systems~\citep{chen2024pid}. 
For such systems, sampling-based model predictive control has been advanced by leveraging the parallel computing capabilities of GPUs~\citep{williams2018information,williams2017model}.
Therefore, we adopt a sampling-based optimal control for LLM alignment by operating in the continuous pre-logit trajectory space.
First, we formulate test-time alignment as a KL-constrained stochastic control problem over pre-logit trajectories.
This formulation yields a reward-tilted optimal distribution through the certain free-energy.
Though this optimal distribution provides an ideal target for alignment, this optimal distribution is intractable because it depends on the reward of generated responses and its normalization constant cannot be computed directly.
To obtain a practical training-free algorithm, we introduce a tractable Gaussian proposal family over pre-logit perturbation trajectories and approximate the optimal distribution by importance sampling. 
This choice enables closed-form density ratios and weighting functions while preserving the ability to actively explore high-reward responses.
The mean of the Gaussian proposal converges to the first moment of the optimal distribution via importance sampling.
 We iteratively update this mean by adaptive importance sampling~\citep{kloek1978bayesian,cappe2004population,bugallo2017adaptive} 
 because na\"ive importance sampling can require a large number of effective samples due to the vast pre-logit sequence space. 
Therefore, our method is called adaptive importance sampling on pre-logits (AISP, \rfig{Fig:illust}).
After explanation of AISP,
we discuss why a Gaussian proposal is a practical choice in pre-logit space, including its connection to the last softmax layer.
Additionally, we reveal that AISP becomes equivalent to BoN with the specific sampling strategy in the limit of a hyperparameter.
Experiments demonstrate that AISP increases reward values 
faster than BoN in terms of the number of used generated samples.
Additionally, AISP also outperforms RE-Control even though it does not require training dataset collection in advance.
Since AISP requires fewer samples than BoN, 
we also evaluate Batched AISP, which simultaneously handles multiple prompts with small samples,
and confirm that Batched AISP can outperform BoN under the same iterations.

\begin{figure*}[tb]
    \centering
    \includegraphics[width=.75\linewidth]{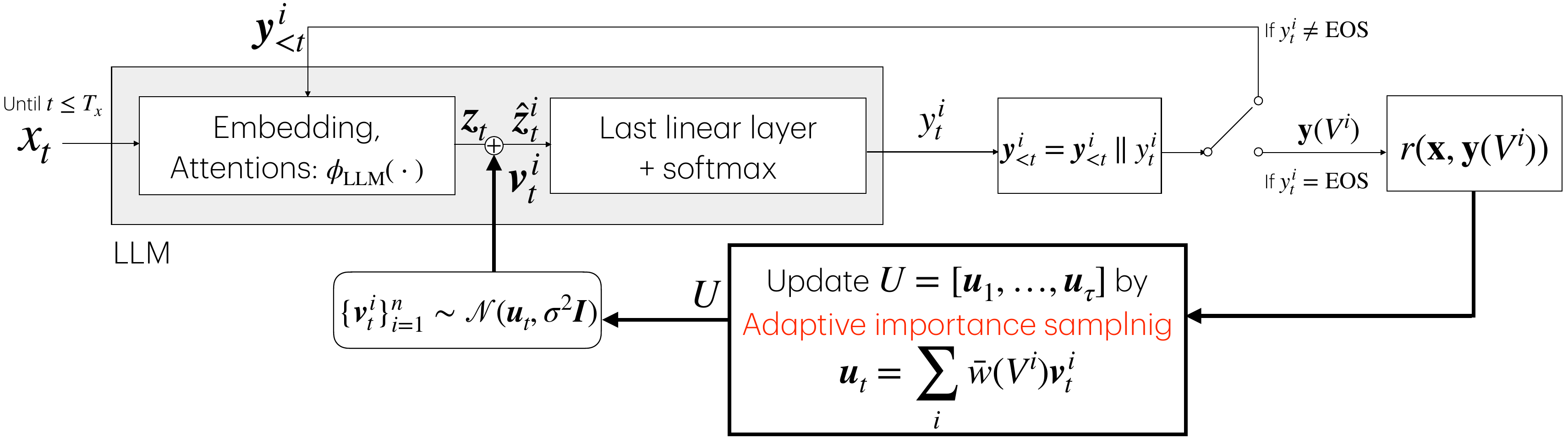}
    \caption{Illustration of AISP. $n$ input trajectries $\{\{\bm{v}^i_t\}_{t=1}^{\tau}\}_{i=1}^{n}$ are sampled from $\mathcal{N}(\bm{u}_t, \sigma^2 \bm{I})$. The input $\bm{v}^i_t$ is added to the pre-logit $\bm{z}_t$, which is obtained by applying LLMs to the past tokens $\bm{y}^i_{<t}$.
    The $t$-th token $y_t^i$ is sampled and concatenated with the past tokens $\bm{y}^i_{<t}$.
     When $y^i_t$ is the end-of-sequence token, the rewards of $\{\mathbf{y}(V^i)\}^{n}_{i=1}$ are evaluated and used in adaptive importance sampling for $\bm{u}_t$.}
    \label{Fig:illust}
\end{figure*}

\section{Preliminary}
\subsection{Best-of-N sampling}
Let $x_t, y_t\in\mathcal{V}$ denote tokens in a vocabulary space $\mathcal{V}$ at the $t$-th position.
Given an input prompt $\x=[x_1,\dots,x_{T_x}]$, an LLM generates a response $\y=[y_1,\dots,y_{T_y}]$ from
the probability $P_{\mathrm{LLM}}(\cdot|\x)$.
Best-of-N sampling (BoN) attempts to generate aligned responses based on a given reward model $r(\x,\y)\in\mathbb{R}$.
BoN samples $N$ responses from the base LLM as $\y\sim P_{\mathrm{LLM}}(\cdot|\x)$ and constructs a set $\mathcal{Y}_N=[\y^1,\dots,\y^N]$.
Next, BoN selects the best sample from the set $\mathcal{Y}_N$ as 
\begin{align}
    \y_{\mathrm{BoN}}=\argmax_{\y\in \mathcal{Y}_N} r(\x,\y).
\end{align}
This simple algorithm is an effective and popular alignment method~\citep{lightman2023let,snell2024scaling,sessa2025bond}.
\citet{yang2024asymptotics} have shown that BoN asymptotically optimizes the following objective function of KL-constrained RL:
\begin{align}
    \max_{\pi(\cdot|\x)} \mathbb{E}_{\y\sim \pi(\cdot|\x)} r(\x,\y)-\lambda D_{KL}(\pi(\cdot|\x)|P_{\mathrm{LLM}}(\cdot|\x))\label{RLObj}.
\end{align}
The KL divergence $D_{KL}(\pi(\cdot|\x)|P_{\mathrm{LLM}}(\cdot|\x))$ prevents $\pi(\cdot|\x)$ from moving far away from the base LLM $P_{\mathrm{LLM}}(\cdot|\x)$.
Equation~(\ref{RLObj}) has a closed solution~\citep{beirami2024theoretical,korbak2022rl,go2023aligning}: 
\begin{align}\textstyle
    \pi^*(\y|\x)=\frac{1}{\eta }P_{\mathrm{LLM}}(\y|\x)\mathrm{exp}(\frac{1}{\lambda}r(\x,\y)),\label{OptEn}
\end{align}
where $\eta\!=\!\sum_{\y} P_{\mathrm{LLM}}(\y|\x)\mathrm{exp}(\frac{1}{\lambda}r(\x,\y))$.
$\eta$ is a normalization constant,
which is hard to estimate~\citep{rafailov2023direct}.
\subsection{RE-Control}
\citet{kong2024aligning} have formulated the LLM alignment as the optimal control problem where the control input $\bm{u}_t\!\in\!\mathbb{R}^d$ is added to the representation of an auto-regressive LLM as
\begin{align}
    y_t\sim
        \mathrm{softmax}(\bm{W}_{\mathrm{LLM}}(\bm{z}_t+\bm{u}_t)+\bm{b}_{\mathrm{LLM}}).\label{RE-Control}
\end{align}
where $\bm{W}_{\mathrm{LLM}}$ and $\bm{b}_{\mathrm{LLM}}$ are the parameters of the last linear layer of the LLM.
$\bm{z}_t\in\mathbb{R}^d$ is called \textit{pre-logit}\footnote{Though $\bm{z}_t$ is called logit in \citep{kong2024aligning}, we call it pre-logit to distinguish it from the input to the softmax function: $\bm{W}\bm{z}_t\!+\!\bm{b}$.}, which 
is the output vector of the penultimate layer of the LLM: $\bm{z}_t\!=\!\phi_{\mathrm{LLM}}(\bm{y}_{<t})$
where $\bm{y}_{<t}\!=\![y_0,\dots,y_{t-1}]$ is a past token sequence including the input prompt $\mathbf{x}$. 
$\phi_{\mathrm{LLM}}(\cdot)$ contains an embedding layer and attention layers. 
In this formulation, $\bm{u}_t$ is optimized through the gradient ascent to maximize the value function $V(\bm{z}_t)$,
which evaluates the current state of LLMs to maximize rewards at the terminal.
However, this value function needs to be trained on the dataset,
which is composed of various states, responses, and rewards: $D_V\!=\!\{(\bm{z}_{0:T}^i,\mathbf{y}^i,r(\mathbf{x}^i,\mathbf{y}^i) )\}_{i=1}^M$.
In fact, the study of \citep{kong2024aligning} uses 349,000 prompts in SHP~\citep{pmlr-v162-ethayarajh22a}
to collect them, which incurs storage costs and training time.\looseness=-1
\section{Proposed method: AISP}
To maximize rewards, we consider applying the control theory to LLM alignment similar to \cite{kong2024aligning}, but without training. 
Traditional optimal control methods do not require any training process because the optimal input trajectories are derived by solving differential equations such as the Pontryagin’s maximum principle. 
However, such methods are ineffective for LLM alignment because LLMs are nonlinear large-scale systems~\citep{chen2024pid}.
For such systems, recent optimal control methods have incorporated a sampling-based approach with model predictive control by considering stochastic input.
Thus, we adopt the stochastic optimal control method 
called model predictive path integral control (MPPI)~\citep{williams2018information,williams2017model} to LLM alignment.
First, we formalize our problem and explain the closed solution.
Since it is an intractable distribution,
we present adaptive importance sampling to solve this problem.
Next, we discuss the connection between the assumption in pre-logit and softmax function,
and the connection with BoN.
Finally, we explain the details of implementation.
\subsection{Problem formulation and optimal distribution in pre-logit space}
Whereas RE-Control~(\req{RE-Control}) uses  the deterministic input $\bm{u}_t$, we apply the stochastic control input $\bm{v}_t$
to LLMs and optimize the distribution of $\bm{v}_t$.
Specifically, we inject a  noise $\bm{v}_t\!\in\!\mathbb{R}^d$ to pre-logit $\bm{z}_t\!=\!\phi_{\mathrm{LLM}}(\bm{y}_{<t})$ for the time interval $t\!\in\![1,\tau]$
where $\tau$ is practically set to maximum generation length.
The input prompt $\mathbf{x}$ corresponds to $\bm{y}_{<1}$.
Then, the $t$-th token is given by
\begin{align}
    \!y_t&\!=\!
        \argmax_i\!\left[\mathrm{softmax}(\bm{W}_{\mathrm{LLM}}(\bm{z}_{t}+\bm{v}_t)+\bm{b}_{\mathrm{LLM}})\right]_i.    \label{sampy}
\end{align}
We consider the distribution of the input trajectory $V\!=\![\bm{v}_1,\dots,\bm{v}_{\tau}]\!\in\!\mathbb{R}^{d\times \tau}$
following the distribution $\mathbb{Q}$ with the density function $q(V)$.
Similar to the objective of KL-constrained RL in \req{RLObj},
we optimize the expected reward values with the KL constraint as 
\begin{align}\label{ObjF}
    \min_{\mathbb{Q}} J(\x,\mathbb{Q})
    =\min_{\mathbb{Q}}-\mathbb{E}_{V\sim \mathbb{Q}}\!&\left[r(\x,\y(V))\right]+\lambda\mathbb{D}_{\mathrm{KL}}(\mathbb{Q}|\mathbb{P}),
\end{align}
where $\y(V)\!=\![y_1,\dots,y_{T_y}]$ is a response generated by \req{sampy}.
$\lambda\mathbb{D}_{\mathrm{KL}}(\mathbb{Q}|\mathbb{P})$
is the regularization term so that the resulting distribution does not deviate from the base LLM
where $\lambda\!>\!0$ is a hyper-parameter and $\mathbb{P}$ is a reference distribution with the density function $p(V)$ for this purpose.
By using the optimal $\mathbb{Q}^*$, we sample the response $\y(V)$ for $V\sim\mathbb{Q}^*$.

To optimize \req{ObjF}, we consider the following free energy~\citep{williams2018information,williams2017model}:
\begin{align}\label{FreeEq}\textstyle
    F(r,p,\x,\lambda)=\mathrm{log}\left( \mathbb{E}_{V\sim\mathbb{P}}\left[ \mathrm{exp}\left(\frac{1}{\lambda} r(\x,\y(V)) \right)\right]\right).
\end{align}
By using Jensen's inequality, we have the following result:
\begin{theorem}
    Free energy in \req{FreeEq} satisfies $-\lambda F(r,p,\x,\lambda) \leq J (\x,\mathbb{Q})$    and the equality holds if\looseness=-1
    \begin{align}\label{OptQ}\textstyle
        q^*(V)=\frac{1}{\eta}\mathrm{exp}\left(\frac{1}{\lambda} r(\x,\y(V))\right)p(V),
    \end{align}
    where $\eta$ is a normalization constant given by
        $\eta=\int_{\mathbb{R}^{d\times \tau}}\mathrm{exp}\left(\frac{1}{\lambda} r(\x,\y(V))\right)p(V)dV$.
\end{theorem}
All of the proofs can be found in Appendix~\ref{sec:proof}.
This theorem shows that the free energy in \req{FreeEq}
is the lower bound of the objective function in \req{ObjF},
and the optimal density function is given by \req{OptQ}. 
However, $q^*(V)$ is intractable and difficult to obtain directly
due to the integral in $\eta$, dependence on rewards, and 
the high-dimensionality of the trajectories $V$.
Therefore, we show how to approximate it by using a tractable distribution and importance sampling.

Note that while the similar result (\req{OptEn}) has been presented for the distribution of \textit{discrete} responses,
\req{OptQ} is related to the distribution of \textit{continuous} pre-logits.
This difference enables to use a tractable approximation and to derive a simple algorithm
as described in the next subsection.

\subsection{Gaussian approximation via adaptive importance sampling}
Since the optimal distribution $q^*(V)$ is intractable,
we approximate it by a tractable proposal distribution via importance sampling, 
which is a commonly used technique in the optimal control method called model predictive path integral control (MPPI)~\citep{williams2018information,williams2017model}.
Specifically, we approximate $q^*(V)$ by a fixed-variance Gaussian distribution and optimize its mean $U$.
First, we consider that $\bm{v}_t$ follows $\mathcal{N}(\bm{u}_t,\sigma^2\bm{I})$
where $\sigma^{2}\!\in\!\mathbb{R}$ is a fixed variance. 
This can be interpreted as the distribution of the pre-logit $\hat{\bm{z}}_t$ is given by $p(\hat{\bm{z}}_t|\bm{y}_{<t})\!=\!\mathcal{N}(\phi_{\mathrm{LLM}}(\bm{y}_{<t})+\bm{u}_t,\sigma^2\bm{I})$ for $t\!\in\![1,\tau]$.
The distribution of the input trajectory $V\!=\![\bm{v}_1,\dots,\bm{v}_{\tau}]\!\in\!\mathbb{R}^{d\times \tau}$
 is a joint Gaussian distribution: \looseness=-1
\begin{align}\label{qDef}\textstyle
    q(V|U,\!\sigma^2)\!=\!\frac{1}{(2\pi\sigma^2)^{\frac{d\tau}{2}}}\mathrm{exp}\!\left(\!-\frac{\sum_{t=1}^{\tau}(\bm{v}_t-\bm{u}_{t})^{\top}(\bm{v}_t-\bm{u}_{t})}{2\sigma^2}\!\right)\!,
\end{align}
where $U=[\bm{u}_1,\dots,\bm{u}_{\tau}]\in\mathbb{R}^{d\times \tau}$ is the mean of the input trajectory.
Additionally, we choose the reference distribution
$\mathbb{P}$ as a zero-mean Gaussian distribution: 
\begin{align}\label{pDef}
    \textstyle p(V|\bm{0},\sigma^2)=\frac{1}{(2\pi\sigma^2)^{\frac{d\tau}{2}}}\mathrm{exp}\left(-\frac{1}{2\sigma^2}\sum_{t=1}^{\tau}\bm{v}_t^{\top}\bm{v}_t\right).
\end{align}
$\mathbb{P}$
 serves as a tractable reference distribution that anchors the controlled generation around the base LLM:
 the KL regularizer penalizes large deviations from the unperturbed pre-logit trajectory.
This mitigates reward hacking problem.
Let $\mathbb{Q}_{U,\sigma^2}$ be the distribution corresponding to the density function $q(V|U,\sigma^2)$.
The KL-divergence $\mathbb{D}_{\mathrm{KL}}(\mathbb{Q}_{U,\sigma^2}|\mathbb{P})$ is given by
$\mathbb{D}_{\mathrm{KL}}(\mathbb{Q}_{U,\sigma^2}|\mathbb{P})\!=\!1/2\sigma^2\sum_{t=1}^{\tau}\!\bm{u}^{\top}_t\bm{u}_t$. 

Then, we consider to approximate the optimal distribution of \req{OptQ} by \req{qDef}
through importance sampling~\citep{robert1999monte,kloek1978bayesian,cappe2004population,bugallo2017adaptive}.
Let $\mathbb{Q}^*$ be the distribution corresponding to $q^*(V)$.
We recall the following theorem, which was established in \cite{williams2018information}:
\begin{theorem}{\citep{williams2018information}}\label{Them:IS}
The KL divergence $\mathbb{D}_{\mathrm{KL}}(\mathbb{Q}^*|\mathbb{Q}_{U,\sigma^2})$
is minimized by $U^*\!=\![\bm{u}^*_1,\dots,\bm{u}^*_\tau]$ where 
\begin{align}\label{OptMean}\textstyle
    \bm{u}^*_t=\mathbb{E}_{V\sim\mathbb{Q}^*}[\bm{v}_t].
\end{align}
Let $q(V|\hat{U},\sigma^2)$ and $\mathbb{Q}_{\hat{U},\sigma^2}$ be a proposal density function for importance sampling and the corresponding distribution, respectively.
Equation~(\ref{OptMean}) is re-written as $\mathbb{E}_{V\sim\mathbb{Q}^*}[\bm{v}_t]=\mathbb{E}_{V\sim \mathbb{Q}_{\hat{U},\sigma^2}}[w(V)\bm{v}_t]$,
where $w(V)$ is the weight function given by\looseness=-1
\begin{align}\textstyle
    w(V)
    \!=\!\frac{1}{\eta}\mathrm{exp}\!\left(\frac{1}{\lambda} r(\x,\y(V))\!-\!\frac{1}{\sigma^2}\!\sum_{t=1}^{\tau}\hat{\bm{u}}_t^{\top}\bm{v}_t\!+\!\frac{1}{2\sigma^2}\sum_{t=1}^{\tau}\hat{\bm{u}}_t^{\top}\hat{\bm{u}}_t\! \right)\!.\label{Rweight}
\end{align}
\end{theorem}
This theorem indicates that the optimal mean $U^*$ is given by importance sampling where the weight function is \req{Rweight}.
 In other words, the mean $\hat{U}$ of the Gaussian proposal distribution
 converges to the first moment of $q^*(V)$ via importance sampling.
We generate $n$ samples $\{V^i\}_{i=1}^{n}$ from the proposal distribution $\mathbb{Q}_{\hat{U},\sigma^2}$
and approximate $U^*$ as $\hat{\bm{u}}^*_t\!=\!\sum_i(w(V^i)\!/\!\sum_j w(V^j)\bm{v}^i_t)\!=\!\sum_i\bar{w}^i\bm{v}_t^i$ 
where $\sum_j w(V^j)$ is empirical normalization instead of $\eta$.
The $i$-th weight $\bar{w}^i$ is given by \looseness=-1
\begin{align}\textstyle
   \bar{w}^i=\frac{\mathrm{exp}\left(\frac{1}{\lambda} r(\x,\y(V^i)) -\frac{1}{\sigma^2}\sum_{t=1}^{\tau}\hat{\bm{u}}_t^{\top}\bm{v}^i_t\right)}{\sum_j\mathrm{exp}\left(\frac{1}{\lambda} r(\x,\y(V^j)) -\frac{1}{\sigma^2}\sum_{t=1}^{\tau}\hat{\bm{u}}_t^{\top}\bm{v}^j_t\right)}\label{WeightEq},
\end{align}
which is implemented by using a softmax function.
In this computation, the term of $\hat{\bm{u}}_t^{\top}\hat{\bm{u}}_t$ in \req{Rweight} is canceled between the numerator and denominator.
Since importance sampling requires a lot of samples if the proposal distribution $q(V|\hat{U},\sigma^2)$ is far from $\mathbb{Q}^*$,
we exploit the adaptive importance sampling~\citep{cappe2004population,bugallo2017adaptive}, which is iterative importance sampling.
Specifically, we updates $\hat{\bm{u}}_t$ by using importance sampling with $n$ sample for $\kappa$ iterations:
\begin{align}\textstyle
    \hat{\bm{u}}_t^{k+1}=\sum_{i=1}^n\bar{w}^i\bm{v}_t^{i,k}, ~~&\bm{v}_t^{i,k}\sim \mathcal{N}(\hat{\bm{u}}^k_t,\sigma^2\bm{I}),~~~~\mathrm{for}~~k\!=\!1,\dots,\kappa.\label{uupdate} 
\end{align}
After the optimization,
we can obtain the optimized response as $\y_{\mathrm{AISP}}\!=\!\y(U^{\kappa+1})$.
Instead of using $\y(U^{\kappa+1})$,
we can generate $\mathcal{V}\!=\!\{V^i|V^i\!\sim\!q(V|U^{\kappa+1},\sigma^2)\}$ and select $\y_{\mathrm{AISP}}\!=\!\mathrm{arg}\!\max_{V\in\mathcal{V}} \y(V)$.
Since the last generation for $U^{\kappa+1}$ is equivalent to increasing the number of iterations by one,
we skip the generation for $U^{\kappa+1}$ in practice.
Similar to BoN, we select the best sample from all samples in the computation of AISP: $\y_{\mathrm{AISP}}\!=\!\y(V^{i^*,k^*})$ where $i^*,k^*\!=\!\mathrm{arg}\!\max_{i,k} r(\x,\y(V^{i,k}))$.
The algorithm of AISP can be found in Appendix~\ref{sec:Alg}.

Note that our formulation is followed by MPPI~\citep{williams2018information}
and can actually be extended to MPPI by changing sampling methods.
While MPPI moves prediction and control windows by determining and applying $\hat{\bm{u}}_1$ to the control target for each iteration, AISP uses a fixed control window $t\!\in\![1,\tau]$.
This is because once moving windows fix the prefix tokens, generated responses lose diversity.
AISP can explore large response spaces by using the fixed window starting $t\!=\!1$ and adaptive importance sampling.\looseness=-1
\subsection{Modeling pre-logits distributions by Gaussian distributions}
As explained above, AISP uses a Gaussian distribution as the proposal distribution.
In this section, we discuss the reason why this proposal reduces the difficulty in the optimization,
and the connection between the Gaussian proposal and the output softmax layer.
\paragraph{How Gaussian assumption simplifies the problem}
 Theorem~\ref{Them:IS} indicates that the Gaussian proposal derives a simple algorithm by using importance sampling.
This is because the KL divergence between the optimal and the Gaussian distributions
is minimized by the expectation of $\bm{v}_t$ over $\mathbb{Q}^*$, and the weight function becomes simple because input trajectories also follow the Gaussian.
If we impose no constraints on the prior distribution,
this computation requires modeling techniques for complicated distributions such as normalizing flows~\citep{power2023variational} and does not yield a simple method.
Though the MPPI literature has explored several extensions to address multimodal distributions~\citep{lambert2021stein} and temporally correlated perturbations~\citep{lee2025time},
we adopt a simple Gaussian proposal as an initial step toward test-time alignment. 
This choice keeps AISP computationally tractable and allows us to isolate the effect of adaptive importance sampling in pre-logit space.
Exploring more expressive proposal families is an interesting direction for future work.
\paragraph{Connection with softmax output layer}
The Gaussian proposal is related to the implicit assumption of pre-logits by the softmax layer.
An auto-regressive LLM generally uses a softmax function with a linear layer as the last layer:
\begin{align}\textstyle
    P_{\mathrm{LLM}}(y_t=y^i|\bm{y}_{<t})=\frac{\exp(\bm{w}_i^{\top}\bm{z}_t+\bm{b}_i)}{\sum_{j=1}^{|\mathcal{V}|}\exp(\bm{w}_j^{\top}\bm{z}_t+\bm{b}_j)}.
\end{align}
The softmax function is derived when the conditional distribution of pre-logits $p(\bm{z}_t|y_t\!=\!y^i)$ is an exponential family distribution.\footnote{This section considers a conditional distribution given an output token $y_t$ not given the past token sequence $\bm{y}_{<t}$.}
If $p(\bm{z}_t|y_t\!=\!y^i)$ is a Gaussian distribution as
$p(\bm{z}_t|y_t\!=\!y^i)\!=\!\mathcal{N}(\bm{\mu}_{i},\bm{\Sigma})$, 
we have the following equality from Bayes' theorem:
\begin{align}\textstyle
    P(y_t\!=\!y^i|\bm{z}_t)\!
    \textstyle
    =\textstyle\frac{p(\bm{z}_t|y_t=y^i)P(y_t=y^i)}{\sum_{j}^{|\mathcal{V}|}p(\bm{z}_t|y_t=y^j)P(y_t=y^j)}
    =\textstyle\frac{
        \mathrm{exp}(\bm{\mu}_{i}^\top\!\Sigma^{-1}\bm{z}-\frac{1}{2}\bm{\mu}_{i}^\top\!\Sigma^{-1}\bm{\mu}_{i}+\ln\!P(y_t=y^i) )
    }{
        \sum_{j}^{|\mathcal{V}|}
        \mathrm{exp}(\bm{\mu}_{j}^\top\!\Sigma^{-1}\bm{z}-\frac{1}{2}\bm{\mu}_{j}^\top\!\Sigma^{-1}\!\bm{\mu}_{j}+\ln \!P(y_t=y^j))
    }\!.
\end{align}
This function corresponds to the softmax function with a linear layer such that
$\bm{w}_i^\top\!=\!\bm{\mu}_{y^i}^\top\Sigma^{-1}$ and $\bm{b}_i\!=\!-\frac{1}{2}\bm{\mu}_{y^i}^\top\Sigma^{-1}\bm{\mu}_{y^i}\!+\!\ln P(y_t\!=\!y^i)$.
From this result,
 \citet{lee2018simple} assume that the pre-trained neural classifier 
 has the pre-logtis following the  Gaussian distribution given a class label in image recognition.
 Since neural language models also use softmax and cross-entropy loss,
we can hypothesize that 
the trained pre-logit distribution $p(\bm{z}_t|y_t)$ follows a Gaussian distribution, which is empirically evaluated in Section~\ref{EmpAISP}.
From this perspective,
AISP can be regarded as exploring the optimal pre-logit distribution $p(\hat{\bm{z}}_t|y_t^*)$ under the assumption, where the distribution given the optimal response can be decomposed as $p(\hat{\bm{z}}_1,\dots,\hat{\bm{z}}_t|\mathbf{y}^*)\!=\!\prod_t p(\hat{\bm{z}}_t|y_t^*)$. 
Note that neural language models commonly use a Gaussian assumption~\citep{li2021bert}
and Mahalanobis distance~\citep{podolskiy2021revisiting}, which can be derived from the Gaussian assumption, on embeddings.

\subsection{Connection with BoN}
In AISP, $\lambda$ is the temperature parameter in softmax of \req{WeightEq}.
Since softmax is a smoothed approximation of argmax and $\lambda\!>\!0$,
\req{WeightEq} is asymptotically close to argmax when $\lambda\!\rightarrow \!0^+$.
Therefore, we have the following result:
\begin{theorem}
When $\lambda\!\rightarrow\!0^+$ and $\kappa\!=\!1$, AISP becomes BoN with the candidate set $\mathcal{Y}_n$ as
\begin{align}\textstyle
    \mathcal{Y}_n=\{\y(V^i)|V^i\sim q(V|\hat{U},\sigma^2),i=1,\dots,n \}.\label{BoNeq}
\end{align}
\end{theorem}
From this result, AISP can be regarded as a continuous approximation of BoN
with a specific candidate set.
In other words, AISP is a generalization of BoN, and AISP subsumes BoN.
Comparing BoN using \req{BoNeq} with BoN using popular sampling strategies, 
we can evaluate the Gaussian proposal distribution in AISP.
Thus, we empirically evaluate BoN using this candidate set when $\hat{U}\!=\!\bm{O}$
and observed it achieves competitive performance in experiments.

\subsection{Implementation}\label{sec:imp}
We have explained the formulation of AISP. 
In this section, we will explain the technique to enhance the practical performance and parallelism in implementation.
\paragraph{Relaxation of constraints}
As discussed above, $\lambda$ can be regarded as a temperature parameter.
Small $\lambda$ allows deviation from the base LLM, but large $\lambda$ causes numerical instability~\citep{williams2018information}.
To achieve both numerical stability and large penalties,
we relax $\mathbb{P}$ as $p(V|\alpha\hat{U},\sigma^2)$ from $p(V|\bm{0},\sigma^2)$ where $0\!<\!\alpha\!<\!1$
by using the technique in MPPI \citep{williams2018information}. 
$\alpha $ and $\lambda$ control the strength of the KL regularization (Appendix~\ref{app:KL}).
As a result, the weight $\bar{w}^i$ becomes
\begin{align}\textstyle
    \bar{w}^i=\frac{\mathrm{exp}\left(\frac{1}{\lambda} r(\x,\y(V^i)) -\frac{1-\alpha}{\sigma^2}\sum_{t=1}^{\tau}\hat{\bm{u}}_t^{\top}\bm{v}^i_t\right)}{\sum_j\mathrm{exp}\left(\frac{1}{\lambda} r(\x,\y(V^j)) -\frac{1-\alpha}{\sigma^2}\sum_{t=1}^{\tau}\hat{\bm{u}}_t^{\top}\bm{v}^j_t\right)}\label{LastWeightEq}.
\end{align}
\paragraph{Parallelization and Batched AISP}\label{Sec:BatchedAISP}
\begin{figure}[tb]
    \centering
    \includegraphics[width=.7\linewidth]{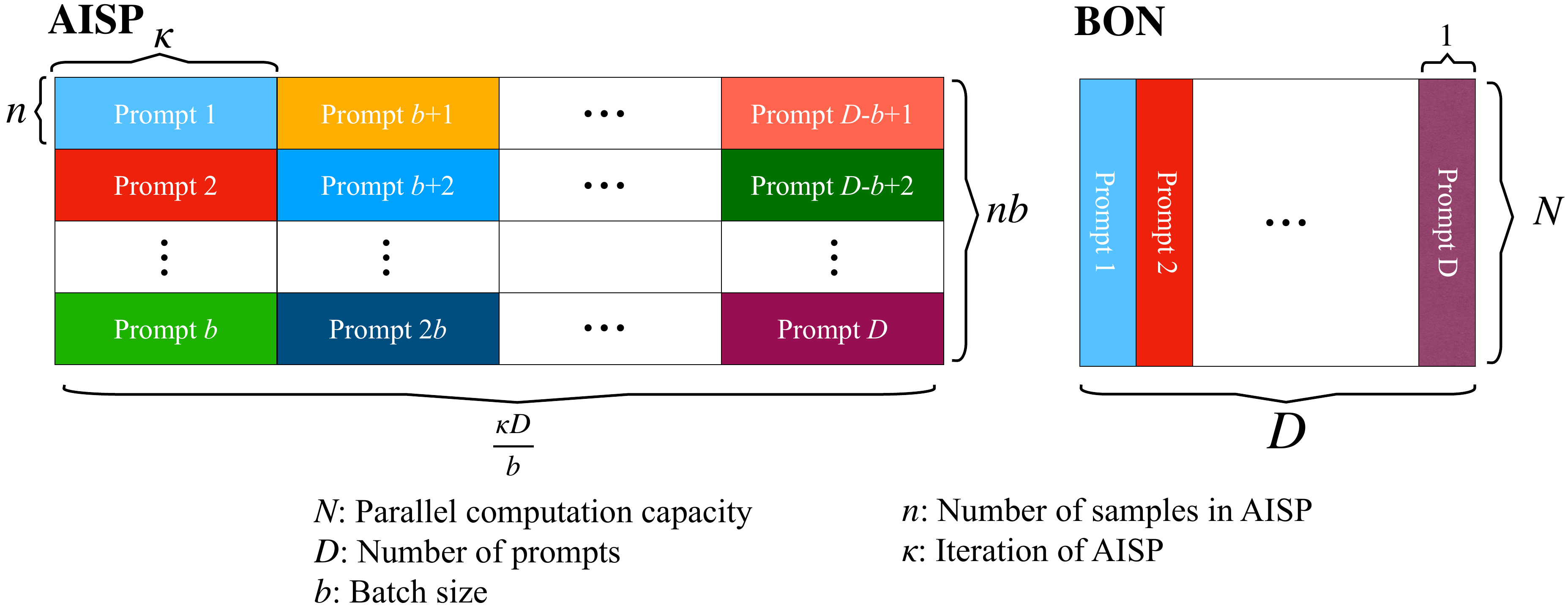}
    \caption{Schematic illustration of computational costs (vertical: parallelism, horizontal: iterations).
    }
    \label{Parallel}
\end{figure}
Adaptive importance sampling contains both parallel and sequential processes.
We generate $\y(V^i)$ for $i=1,\dots,n$ in parallel, like BoN. 
In contrast, $\kappa$ iterations of \req{uupdate} should be executed sequentially.
Therefore, $n$ and $\kappa$ increase space complexity and time complexity, respectively.
They should be tuned according to practical needs and performance.
Additionally, we can compute AISP for Batched prompts $\{\mathbf{x}^i\}_{i=1}^b$.
Let $D$ and $b$ be the number of total prompts and batch size.
The number of iterations and parallel computations in Batched AISP and BoN become almost the same 
when $\kappa=b$ and $n=N/b$ (\rfig{Parallel}). 
Strictly speaking, there is the overhead for computing
$\sum_{t=1}^{\tau}\hat{\bm{u}}_t^{\top}\bm{v}^i_t$ in the weight function (\req{LastWeightEq}) of $O(\tau d)$, which is negligible compared to the overall complexity.
We compare Batched AISP with BoN in Section~\ref{batchedAISP}

\section{Related work}
There are several test-time alignment methods that train auxiliary evaluators such as critics, value functions, or autoregressive reward models,
e.g., RE-Control~\citep{kong2024aligning}, Critic-Guide Decoding~\citep{kim2023critic}, Controlled Decoding~\citep{mudgal2024controlled}, 
and GenARM~\citep{xugenarm}.
Though these methods avoid the high computational costs of fine-tuning LLMs,
they still require additional training, and some need to build datasets in advance~\citep{kong2024aligning,mudgal2024controlled,han2024value}.
Among them, we compare AISP with RE-Control because it is the closest to AISP.
\citet{khanov2024args} proposes ARGS, which adds the weighted reward to the logit for each token.
ARGS can be used as a training-free test-time alignment given a reward model,
and we also compare AISP with it. 
Similarly, transfer-$Q^*$~\citep{chakraborty2024transfer} estimates token-level value function 
through the trajectory-based reward function. When the base LLM is not aligned with the given target reward in advance,
transfer-$Q^*$ requires the base reward model, and it is hard to fairly compare it with AISP.

Another line of work studies sampling-based alignment, especially BoN and its variants~\citep{snell2024scaling,lightman2023let,brown2024large,ichiharaevaluation,jinnai2024regularized}.
\citet{snell2024scaling} have investigated the computational cost in test-time alignments and
showed that BoN outperforms beam-search when using a high computational budget.
While it is revealed that BoN's win-rate against a base LLM is bounded by $N/(N+1)$~\citep{beirami2024theoretical}, 
there are few studies to improve the efficiency of reward optimization in terms of the number of used samples.
As another sampling method, \citet{zhu2025soft} proposed Soft Reasoning based on Bayesian optimization.
Soft Reasoning applies Bayesian optimization by using the Gaussian perturbation similar to AISP.
However, Soft Reasoning only applies the perturbation to the initial token embedding,
which might be due to the difficulty in Bayesian optimization.
AISP operates on a more expressive sequence because it explores the entire trajectory of pre-logits.
\citet{loula2025syntactic} used importance sampling to control generation of LLMs on the basis of the given potential function.
Though it is similar to AISP, they generate tokens using task-specific potential functions rather than the reward model.
Additionally, they use importance sampling on the token space rather than the pre-logit space.
Thanks to continuous pre-logits space, our method can 
employ a Gaussian distribution, which is easy to handle.\looseness=-1
\section{Experiments}
\subsection{Setup}
\paragraph{Datasets and models}
We conduct experiments to evaluate the effectiveness of AISP
on test-time alignment of LLMs for helpfulness and minimizing harmfulness.
We use Anthropic’s HH-RLHF~\citep{bai2022training} and Stanford human preferences (SHP) datasets~\citep{pmlr-v162-ethayarajh22a}
following \citep{kong2024aligning}.
These datasets are used to align LLMs for helpfulness and harmlessness.
We use randomly selected 1000 entries of the test datasets due to limited computation resources, like \citep{jinnai2024regularized}.
We use Llama-3-8B~\citep{llama3modelcard}, Vicuna-7B-v1.5~\citep{vicuna2023}, 
and Gemma3-4B~\citep{gemma_2025} as the base LLMs, 
and reward models are UltraRM-13b (UltraRM)~\citep{cui2023ultrafeedback} and Eurus-RM-7b (Eurus)~\citep{yuan2024advancing}.
In Appendix~\ref{sec:othertask}, 
we evaluate AISP on Alpaca-Eval~\citep{Li_AlpacaEval_An_Automatic_2023}, GSM8K~\citep{cobbe2021gsm8k}, HumanEval~\citep{chen2021codex}, and TruthfulQA~\citep{lin2022truthfulqa}.
\paragraph{Baselines and hyper-parameter tuning} 
We compare AISP with BoN using two generating strategies to construct candidate sets $\mathcal{Y}_N$: \textit{BoN (top-p)} and \textit{BoN ($\mathcal{N}$)}.
BoN~(top-p) uses top-p (nucleus) sampling, and BoN~($\mathcal{N}$) injects
$n\kappa$ Gaussian noises following $\mathcal{N}(\bm{0},\sigma^2\bm{I})$ into pre-logits as \req{BoNeq}.
After injection, we use greedy search for each injected pre-logits to construct $\mathcal{Y}_N$ of  BoN~($\mathcal{N}$).
Both $n$ and $\kappa$ of AISP are set to 32, and $N$ of BoN is set to 1024 ($=\kappa n$).
Additionally, we also compare AISP with ARGS-greedy (ARGS)~\cite{khanov2024args} and RE-Control~\cite{kong2024aligning},
which are reward-based test-time alignment methods.
ARGS adds the weighted reward to logits and selects the best token for each token generation,
and RE-Control adds the control input to maximize the value function.
We tune hyper-parameters for each method on partial training datasets,
which is described in Appendix~\ref{App:Hyp}.
Sensitivity to hyper-parameters in AISP is shown in Appendix~\ref{App:HypDep}

\paragraph{Evaluation metrics} The evaluation metrics are the reward values and win rate against BoN.
We evaluate $r(\x,\y)$ at the last by using UltraRM, and win rate is the rate at which GPT-4 
considers that the response is better than baseline responses following~\citep{kong2024aligning,khanov2024args,vicuna2023}.
While previous studies \citep{kong2024aligning,khanov2024args,vicuna2023} 
use the preferred response as baseline responses, 
we directly compare the responses of AISP with those of BoN.
Additionally, we evaluate diversity and coherence following~\citep{kong2024aligning,khanov2024args} in Appendix~\ref{App:DivAndCoh}.
\subsection{Main results: Average reward, win rate, and convergence}
\begin{table*}[tbp]
    \caption{Average Rewards. For BoN, $N$ is set to $n\kappa$.
    Values are presented as mean (standard deviation) for three trials. ARGS-greedy does not contain a stochastic process.}
    \label{Table:main}
    \centering
    {\scriptsize
    \begin{tabular}[tb]{cccccccc}\toprule
        Dataset &LLM &Reward model &BoN (top-$p$) &BoN ($\mathcal{N}$)&RE-Control&ARGS&AISP (Ours)\\\cmidrule(r){1-3}\cmidrule(r){4-8}
        SHP&Llama3-8B& UltraRM& -2.38 (0.04)&-2.30 (0.03) & -9.28 (0.03)& -3.94 &\textbf{-1.39} (0.02) \\
        SHP&Vicuna-7B & UltraRM& -1.78 (0.02 )&-1.94 (0.01)&-5.67 (0.04) &-11.97 & \textbf{-1.46} (0.02) \\
        SHP&Gemma3-4B& UltraRM& -3.43 (0.02)&-3.29 (0.02)& -9.97 (0.02) & -7.08  & \textbf{-2.39} (0.03)\\
        SHP&Llama3-8B & Eurus&-6.42 (0.08)&-7.22 (0.04)&  -9.62 (0.1)&-11.91& \textbf{-6.17} (0.03)\\
        SHP&Vicuna-7B & Eurus& -3.83 (0.02)&-4.09 (0.02)&-5.24 (0.03)& -12.67&\textbf{-3.72} (0.02)\\
        SHP& Gemma3-4B& Eurus&-6.45 (0.06)&-6.30 (0.04)& -10.1 (0.1) &-14.1& \textbf{-5.78} (0.03) \\\cmidrule(r){1-3}\cmidrule(r){4-8}
        HHRLHF&Llama3-8B& UltraRM&-2.62 (0.003)&-2.60 (0.01)& -7.54 (0.06)& -9.27 & \textbf{-2.45} (0.00)  \\
        HHRLHF&Vicuna-7B & UltraRM & -3.08 (0.00)&-3.02 (0.01)& -4.43 (0.05)& -11.53& \textbf{-2.86} (0.02)   \\
        HHRLHF&Gemma3-4B& UltraRM& -2.60 (0.02)&-2.35 (0.01)& -9.00 (0.2) & -11.11  & \textbf{-2.18} (0.02)\\
        HHRLHF&Llama3-8B & Eurus& -5.02 (0.03)&\textbf{-4.91} (0.03) & -11.28 (0.3) & -10.50& -4.96 (0.05) \\
        HHRLHF&Vicuna-7B & Eurus &-3.71 (0.00)&-3.77 (0.02)& -10.24 (0.2)& -11.46&\textbf{-3.65} (0.04)\\
        HHRLHF& Gemma3-4B& Eurus& -4.82 (0.05)&-4.84 (0.03)& -10.81 (0.3) & -11.74 &\textbf{-4.80} (0.02) \\
        \bottomrule
    \end{tabular}
    }
\end{table*}
Table~\ref{Table:main} lists average rewards of each method.
AISP achieves the highest among methods.
AISP achieved up to about 40\% improvement over BoN (top-$p$).
AISP also outperforms RE-Control even though it does not require building training datasets.
This result indicates that  AISP is superior to baselines as a sampling-based reward optimization.
ARGS does not work very well in our experiment. 
This is because ARGS needs to evaluate next token generation by a reward model.
For this purpose, a token-level reward model $r(y_t, \bm{y}_{<t})$ is more suitable than a trajectory-level reward model $r(\mathbf{x}, \mathbf{y})$ used in our experiment.
However, token-level reward models generally require additional training or specialized techniques~\citep{yoon2024tlcr,chakraborty2024transfer,xugenarm}.
BoN with gaussian pre-logit noise (BoN ($\mathcal{N}$))
achieves competitive performance with BoN (top-$p$). 
This implies that the Gaussian assumption for pre-logits is not very strong,
and it generates good candidate responses even without importance sampling.

Table~\ref{Table:sub} lists the win rate for AISP vs BoN (top-$p$).
To compute win rate, we sampled 100 pairs of prompts and responses at random, and GPT-4 judges whether the response from AISP or BoN is better.
Since the values are averaged over three trials and rounded off, they do not always sum to 100~\%.
This table shows that 
AISP has higher win rates than those of BoN (top-$p$) under almost all of conditions. 
The results of average rewards and win rates show that AISP aligns LLMs better than BoN: 
i.e., AISP can generate more  helpful and harmless responses through maximization of rewards than BoN.\looseness=-1

To compare sample efficiencies of AISP and BoN,
\rfig{Fig} 
plots curves of reward values during iterations on SHP using Llama-3-8B and UltraRM.
In this figure, AISP (Mean at $k$) is $1/n\sum_i r(\x,\y(V^{i,k}))$.
AISP (Best at $k$) is $\max_{i} r(\x,\y(V^{i,k}))$, and AISP (Best so far) is $\max_{i,1\leq j\leq k} r(\x,\y(V^{i,j}))$, which is $r_{\mathrm{best}}$ in Algorithm~\ref{Alg} at $k$.
BoN corresponds to $\max_{\y\in \mathcal{Y}_N} r(\x,\y)$ using $N\!=\!nk$ samples.
These curves are also evaluated on randomly selected 100 pairs, and are averaged over data samples and three trials.
This figure shows that though AISP underperforms BoN in the early iterations,
it improves more rapidly and eventually surpasses BoN as the number of iterations increases.
In addition, while the maximum number of iterations $k$ was set to 32 in this experiment,
it is likely that the performance gap would become more pronounced with a larger number of iterations.
Reward values of AISP (Mean at $k$) and AISP (Best at $k$) also increase according to $k$. 
This indicates that AISP not only optimizes the resulting response but also optimizes the distribution of responses.
Thus, AISP obtains aligned distributions of response without any additional techniques.
Note that results on other models and datasets are presented in Appendix~\ref{sec:subconv}, which show consistent trends with the main results.
\begin{figure}[tb]
    \begin{minipage}{0.52\textwidth}
    \centering
    \captionof{table}{Win rate for AISP vs BoN (top-$p$) on SHP (top) and HHRLHF (bottom):
    (L: Llama3-8B, V: Vicuna-7B, G: Gemma3-4B, U: UltraRM, E: Eurus)}
    \label{Table:sub}
    \centering
    {\scriptsize
    \begin{tabular}[tb]{@{}c@{\hspace{2.0mm}}c@{\hspace{2.0mm}}c@{\hspace{2.0mm}}c@{\hspace{2.0mm}}c@{\hspace{2.0mm}}c@{\hspace{2.0mm}}c@{}}\toprule
        &L \& U &L \& E &V \& U &V \& E &G \&  U &G \& E \\\midrule
AISP&\textbf{51.3}&\textbf{47.0}&\textbf{35.3}&\textbf{36.0}&\textbf{53.0}&\textbf{52.7}\\
Draw&6.7&7.7&30.3&36.0&5.7&8.3\\
    BoN&42.0&45.3&34.3&28.0&41.3&39.0\\\midrule
AISP&\textbf{41.0}&\textbf{46.3}&\textbf{25.0}&32.3&\textbf{44.0}&\textbf{47.3}\\
Draw&18.0&10&50.0&34.7&19.3&11.0\\
BoN&\textbf{41.0}&43.7&\textbf{25.0}&\textbf{33.0}&36.7&41.7\\\bottomrule
\end{tabular}
    }\\\vspace{5pt}
    \centering
    \includegraphics[width=.67\linewidth]{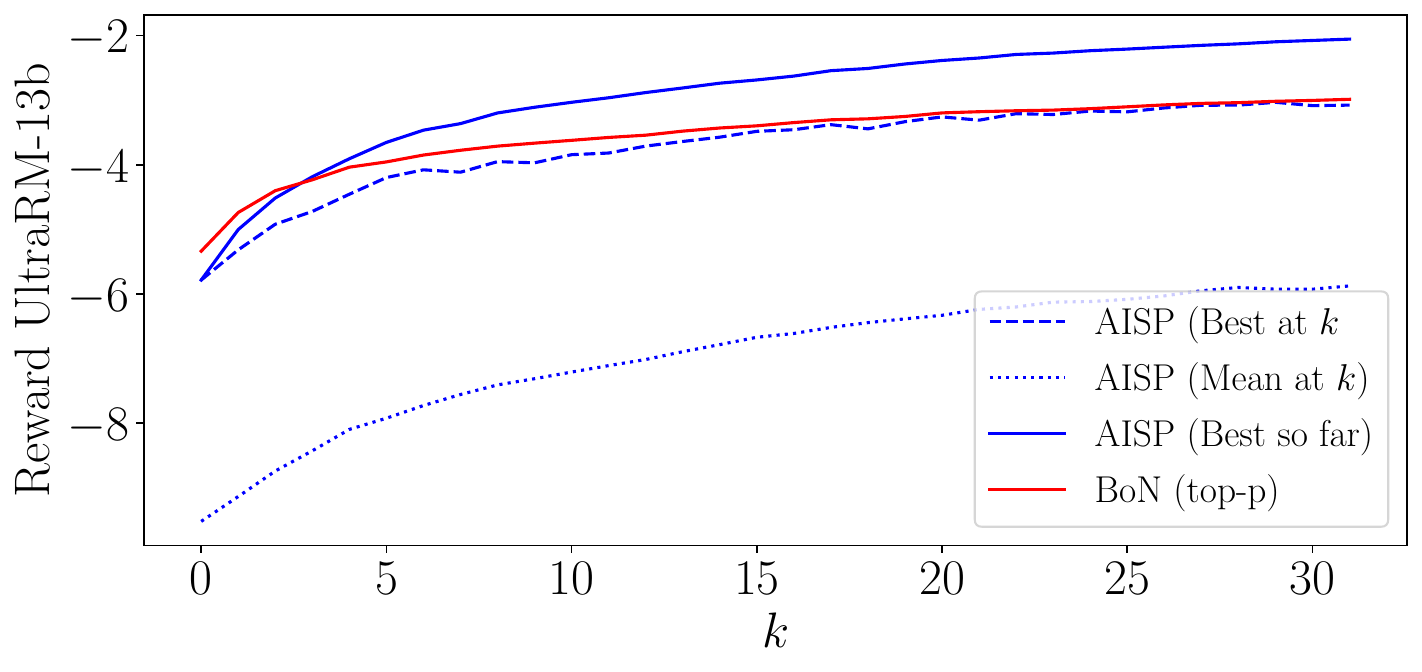}
    \captionof{figure}{Reward curves over iterations. Both methods generate 32 samples/$k$.}
    \label{Fig}
        \centering
\includegraphics[width=.67\linewidth]{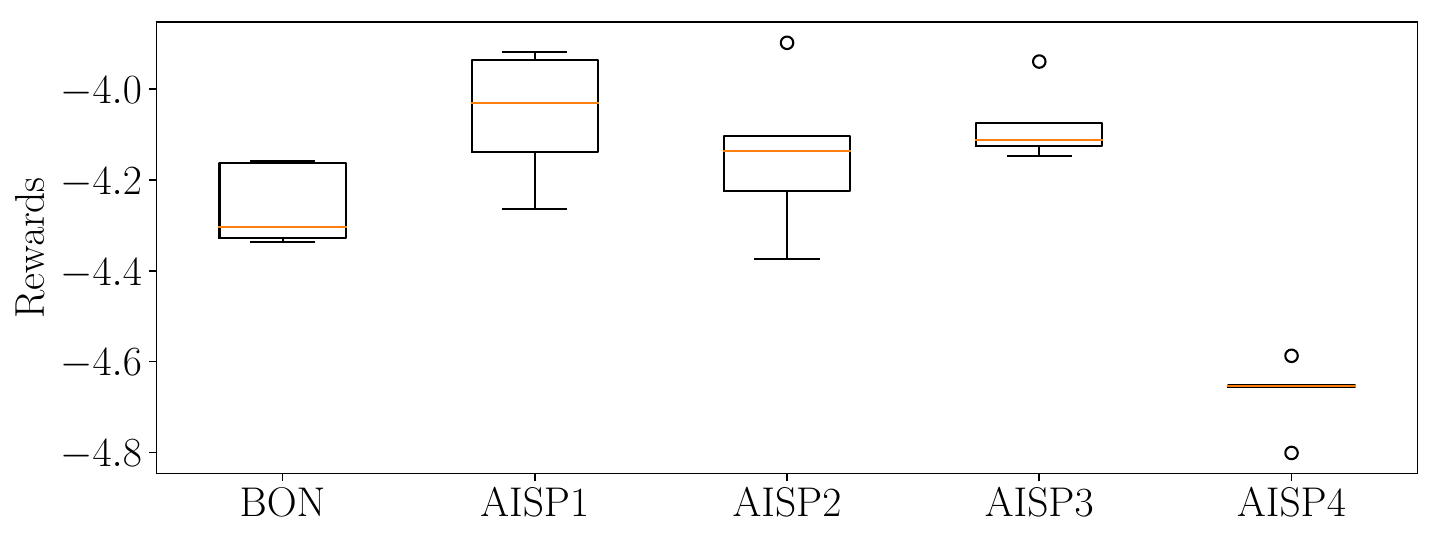}
    \captionof{figure}{Rewards of Batched AISP and BoN. }
    \label{Fig:Batch}

\end{minipage}\hfill
    \begin{minipage}{0.45\textwidth}
    \centering
    \includegraphics[width=.78\linewidth]{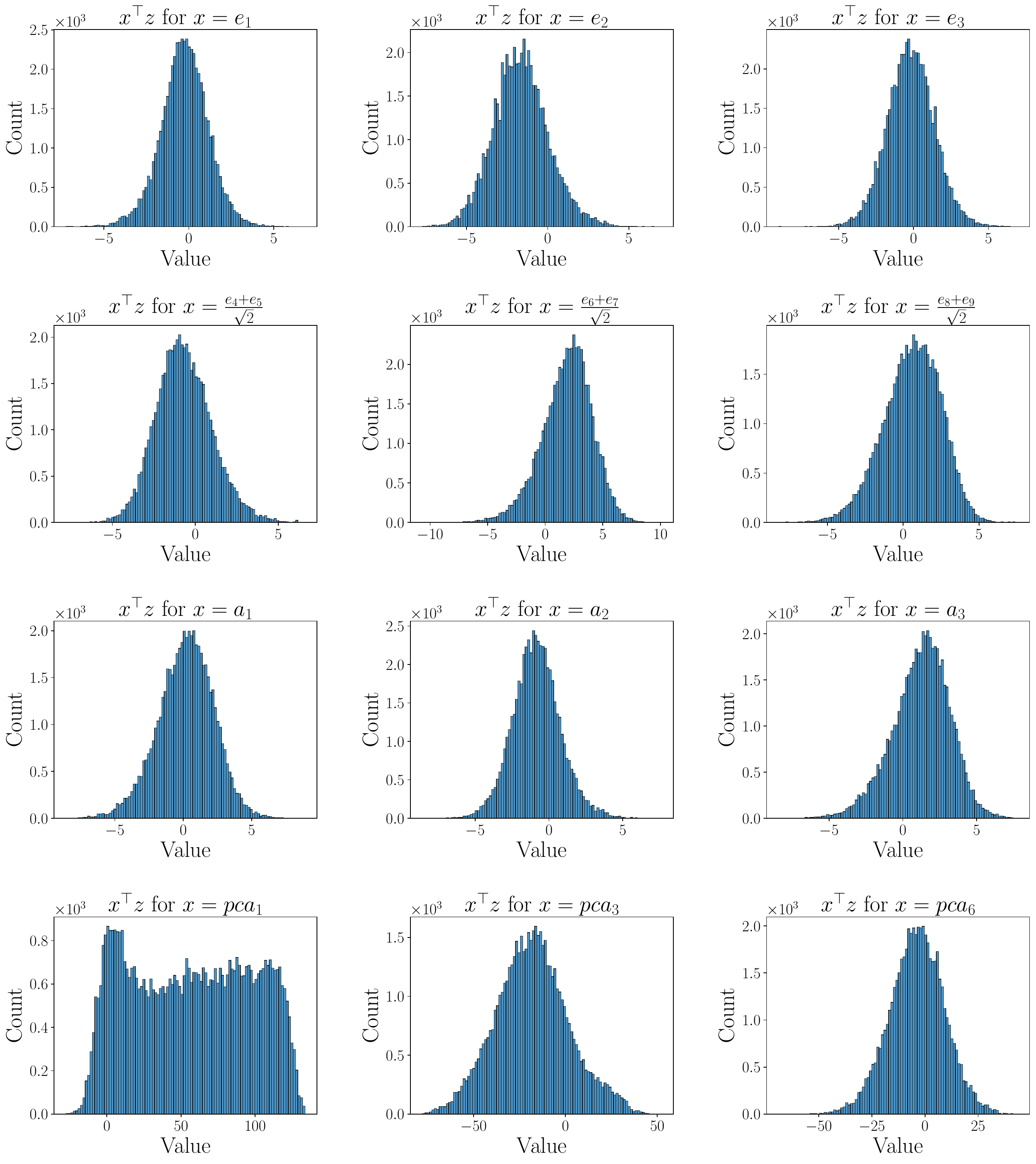}
    \caption{Histogram of projected prelogits}
    \label{Gauss}
    \centering
    \includegraphics[width=.88\linewidth]{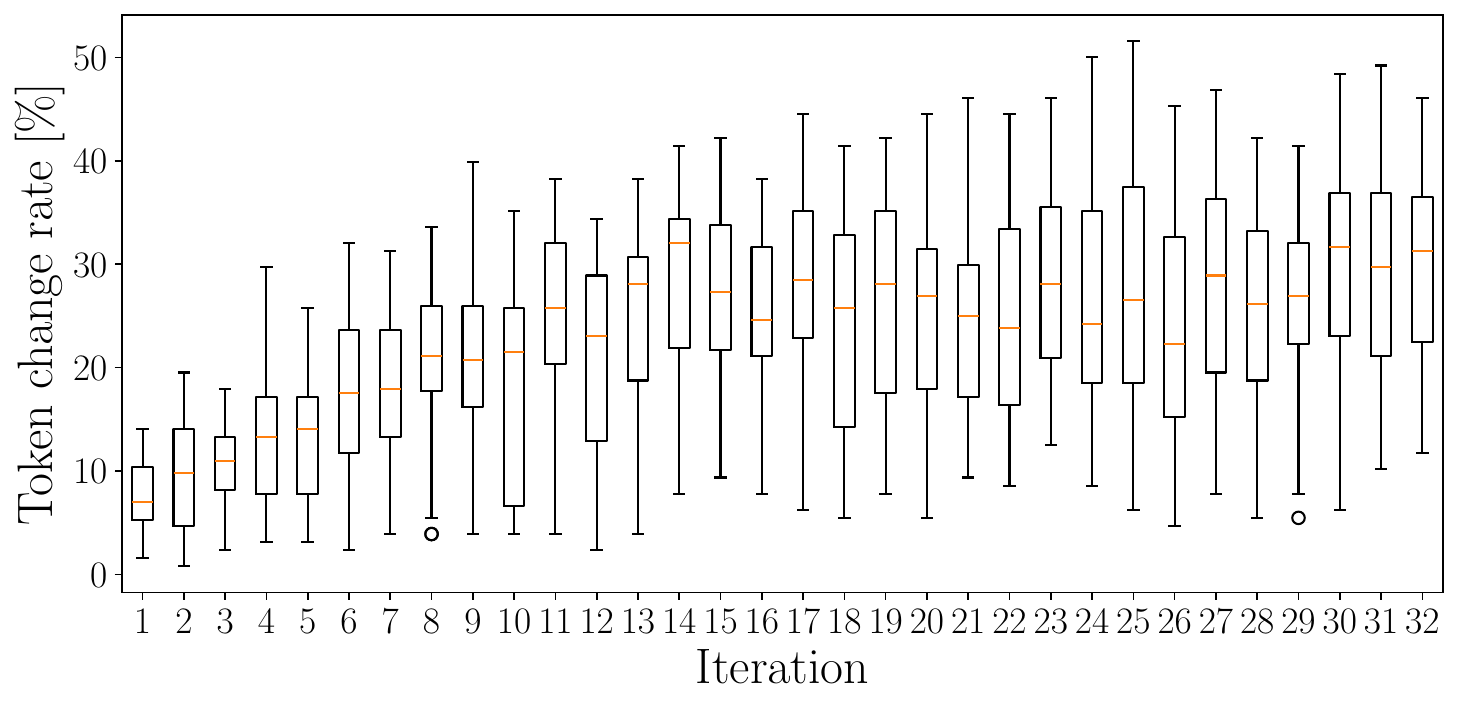}
    \caption{Argmax token flip rate by $\bm{v}_t$.
    At each token, we fix the AISP-generated prefix.}
    \label{TokenChange}
\end{minipage}

\end{figure}
\subsection{Batched AISP}\label{batchedAISP}
As above, AISP achieves higher rewards than BoN with fewer samples.
However, it requires sequential process: $\kappa$ iterations.
When we need to reduce time complexity, 
AISP can be accelerated by processing $b$ prompts with small $n$ in batches as discussed in Section~\ref{sec:imp}.
We compare Batched AISP with BoN of $N=128$ under the same iterations for processing multiple prompts; i.e., $\kappa\!=\!b$ and $N\!=\!nb$. 
In this experiment, we use Llama3-8B with UltraRM on 100 prompts in SHP and evaluate Batched AISP under multiple settings of $(b,n)$.
In \rfig{Fig:Batch}, $(b,n)$ of AISP1, AISP2, AISP3, and AISP4 correspond to (8,16), (16,8), (32,4), and (64,2), respectively.
\rfig{Fig:Batch} demonstrates that AISP can outperform BoN even under the same iterations for $D$ prompts.
In addition, this figure shows that AISP can exceed BoN 
if it has at least four samples per iteration. 

We evaluate the wall-clock time of AISP on a stand-alone server (A100 VRAM 40GB).
For fair comparison, we set the generated token length to fixed 128.
First, to confirm that the overhead for weight updating 
\req{uupdate} is trivial as explained in Section~\ref{Sec:BatchedAISP},
we evaluate the wall-clock time for one iteration of AISP ($n\!=\!32$) and BoN ($N\!=\!32$) averaged over 10 prompts in SHP with Llama3-8B and UltraRM. 
AISP and BoN take 7.75~s and 7.68~s, respectively. 
Thus, the overhead for weight updating is only 1\%.
Next, we evaluated the wall-clock time for Batched AISP $(b, n)\!=\!(4,8)$ and BoN ($N\!=\!32$) on 100 prompts in SHP. 
Batched AISP and BoN take 9.35~s and 6.83~s per prompt, respectively.
The overhead reaches 36.8~\% because the wall-clock time of Batched AISP for one mini-batch is determined by the largest prompts in the mini-batch.
After sorting the prompts in terms of the token length, the wall-clock time of Batched AISP becomes 7.38~s/prompt,  and thus, the overhead is just 8~\%.\looseness=-1

\subsection{Empirical characterization of AISP}\label{EmpAISP}
\paragraph{Distribution of pre-logits given output token}
To justify the Gaussian proposal distribution,
we analyzed $p(\bm{z}_t |y_t)$ for Llama3-8B. We collect $\bm{z}_t$ for the fixed $y_t$ during generation on SHP without noise injection and plot the histogram of one-dimensional projected pre-logit $\lbrace \bm{x}^\top \bm{z}^i_t\rbrace_{i}$ in \rfig{Gauss}.
$y_t$ is set to ``and'' because it is a common word appearing frequently. To visualize high-dimensional $\bm{z}_t$, we projected them to one dimension by $\bm{x}^\top\bm{z}_t$ using a fixed vector $\bm{x}$ since inner products follow Gaussian if $\bm{z}_t$ follows Gaussian.
$\bm{x}$ is set to a randomly selected axis: $\bm{e}_i$, mixed axis: $(\bm{e}_i+\bm{e}_j)/\sqrt{2}$, normalized random vector: $\bm{a}=\bm{x}/||\bm{x}||$ for $\bm{x}\sim \mathcal{N}(0,I)$, and the 1st, 3rd, 6th principal component analysis directions of $\lbrace \bm{z}_t^i\rbrace_i$: $\bm{pca}$.
In \rfig{Gauss}, most of $\bm{x}^\top \bm{z}_t$ follow uni-modal Gaussian-like distributions. For the 1st PCA direction (left bottom), they do not follow Gaussian while it is close to Gaussian from the 3rd PCA onwards. This suggests that the true conditional pre-logit distribution is not exactly Gaussian, but that a Gaussian proposal is a reasonable approximation in many directions.
\paragraph{Token-level effects of pre-logit perturbations}
To quantify the local effect of the pre-logit perturbation $\bm{v}_t$,
we evaluate how often it changes the next token during AISP decoding. 
For each AISP-generated sequence, we fix the prefix $\bm{y}_{<t}$ at each step and 
compare the next-token argmax with and without $\bm{v}_t$. 
The flip rate is the fraction of decoding steps at which the top-logit token changes. 
This is a step-wise diagnostic and does not compare the final AISP response with an unperturbed generated response.
Figure~\ref{TokenChange} shows that 
token flip only occurs in about 10~\% of tokenes for each response in the early iterations.
This indicates that AISP explores good responses near the base LLM early on.
As iterations increase, the frequency of flips increases and enables the generation of low-probability responses relative to the base LLM.
Even after $\bm{u}_t$ has a large value, flips do not always occur.
This indicates that AISP appears to preserve much of the base LLM’s local generation behaviour.\looseness=-1
\section{Conclusion}
In this paper, we propose adaptive importance sampling on a pre-logit distribution
for alignment of LLMs. Our method injects a Gaussian perturbation to the pre-logit of the base LLM,
and optimizes the Gaussian perturbation through importance sampling.
Since our method is simple, 
future work could include combinations of AISP and fine-tuning, and different importance sampling techniques.
\bibliography{CameraBib.bib}
\bibliographystyle{unsrtnat}

\newpage
\appendix
\section{Proofs}\label{sec:proof}
\subsection{Proof of Theorem~3.1}
\begin{theorem*}
    Free energy \req{FreeEq} satisfies $-\lambda F(r,p,\x,\lambda) \leq J(\x,\mathbb{Q})$    and the equality holds if 
    \begin{align}
        q^*(V)=\frac{1}{\eta}\mathrm{exp}\left(\frac{1}{\lambda} r(\x,\y(V))\right)p(V)
    \end{align}
    where $\eta$ is a normalization constant given by
        $\eta=\int_{\mathbb{R}^{d\times \tau}}\mathrm{exp}\left(\frac{1}{\lambda} r(\x,\y(V))\right)p(V)dV$.
\end{theorem*}
\begin{proof}
    Similar to importance sampling, $F$ can be written by using $\mathbb{Q}$ as
    \begin{align}
        &F(r,p,\x,\lambda)\nonumber\\
        =&\mathrm{log}\left( \int  \mathrm{exp}\left(\frac{1}{\lambda} r(\x,\y(V))\right)\frac{q(V)}{q(V)} p(V) dV\right)\\
       =&\mathrm{log}\left( \int  \mathrm{exp}\left(\frac{1}{\lambda} r(\x,\y(V))\right)\frac{p(V)}{q(V)} q(V) dV\right)\\
        =&\mathrm{log}\left( \mathbb{E}_{V\sim\mathbb{Q}}\left[ \mathrm{exp}\left(\frac{1}{\lambda} r(\x,\y(V)) \right)\frac{p(V)}{q(V)}\right]\right)
    \end{align}
    From Jensen's inequality, we have
    \begin{align}
        &F(r,p,\x,\lambda)\nonumber\\
        =&\mathrm{log}\left( \mathbb{E}_{V\sim\mathbb{Q}}\left[ \mathrm{exp}\left(\frac{1}{\lambda} r(\x,\y(V)) \right)\frac{p(V)}{q(V)}\right]\right)\\
        \geq&  \mathbb{E}_{V\sim\mathbb{Q}}\left[\mathrm{log} \left(\mathrm{exp}\left(\frac{1}{\lambda} r(\x,\y(V)) \right)\frac{p(V)}{q(V)}\right)\right]\\
        =&\mathbb{E}_{V\sim\mathbb{Q}}\left[\frac{1}{\lambda} r(\x,\y(V))-\mathrm{log} \left(\frac{q(V)}{p(V)}\right)\right]\\
        =&\mathbb{E}_{V\sim\mathbb{Q}}\left[\frac{1}{\lambda} r(\x,\y(V))\right]-D_{\mathrm{KL}}(\mathbb{Q}|\mathbb{P})
    \end{align}
    Multiplying both sides of each equation by $-\lambda$, we have the following:
    \begin{align}
        -\lambda F(r,p,\x,\lambda)&\leq-\mathbb{E}_{V\sim\mathbb{Q}}\left[r(\x,\y(V))\right]+\lambda D_{\mathrm{KL}}(\mathbb{Q}|\mathbb{P}).
    \end{align}
    Next, we substituting \req{OptQ} into KL divergence as:
    \begin{align}
        &D_{\mathrm{KL}}(\mathbb{Q}^*|\mathbb{P})\nonumber\\
    =&\int \log \left(\frac{q^*(V)}{p(V)}\right)q^*(V)dV\\
    =&\int \log\left( \frac{\frac{1}{\eta}\mathrm{exp}\left(\frac{1}{\lambda} r(\x,\y(V))\right)p(V)}{p(V)}\right)q^*(V)dV\\
    =&\int \log \frac{1}{\eta}\mathrm{exp}\left(\frac{1}{\lambda} r(\x,\y(V))\right)q^*(V)dV\\
    =&-\log(\eta)+\int \frac{1}{\lambda} r(\x,\y(V))q^*(V)dV\\
    =&-\log(\eta)+\frac{1}{\lambda}\mathbb{E}_{V\sim \mathbb{Q}^*} \left[ r(\x,\y(V))\right]. \label{KLEta}
    \end{align}
    $-\log(\eta)$ becomes $-F(r,p,\x,\lambda)$ as
    \begin{align}
        -\log(\eta)&\!=\!-\log\left(\int_{\mathbb{R}^{d\times \tau}}\mathrm{exp}\left(\frac{1}{\lambda} r(\x,\y(V))\right)p(V)dV\right)\\
        &\!=\!-F(r,p,\x,\lambda)
    \end{align}
    Therefore, \req{KLEta} becomes
    \begin{align}
        D_{\mathrm{KL}}(\mathbb{Q}^*|\mathbb{P})&=-F+\frac{1}{\lambda}\mathbb{E}_{V\sim \mathbb{Q}^*} \left[ r(\x,\y(V))\right].
    \end{align}
    and thus, we have
    \begin{align}
        -\lambda F=-\mathbb{E}_{V\sim \mathbb{Q}^*} \left[ r(\x,\y(V))\right]+\lambda D_{\mathrm{KL}}(\mathbb{Q}^*|\mathbb{P}).
    \end{align}
which completes the proof.
\end{proof}
\subsection{Proof of Theorem~3.2}
The results in Theorem~3.2 has been already shown by \citet{williams2018information}.
Even so, we provide the proof to clarify the derivation of AISP.
\begin{theorem*}{\citep{williams2018information}}
The KL divergence $\mathbb{D}_{\mathrm{KL}}(\mathbb{Q}^*|\mathbb{Q}_{U,\sigma^2})$
is minimized by $U^*\!=\![\bm{u}^*_1,\dots,\bm{u}^*_\tau]$ where 
\begin{align}
    \bm{u}^*_t=\mathbb{E}_{V\sim\mathbb{Q}^*}[\bm{v}_t].
\end{align}
Let $q(V|\hat{U},\sigma^2)$ and $\mathbb{Q}_{\hat{U},\sigma^2}$ be a proposal density function for importance sampling and the corresponding distribution, respectively.
Equation~(\ref{OptMean}) is re-written as $\mathbb{E}_{V\sim\mathbb{Q}^*}[\bm{v}_t]=\mathbb{E}_{V\sim \mathbb{Q}_{\hat{U},\sigma^2}}[w(V)\bm{v}_t]$,
where $w(V)$ is the weight function given by\looseness=-1
\begin{align}
    w(V)
    &=\frac{1}{\eta}\mathrm{exp}\left(\frac{1}{\lambda} r(\x,\y(V)) -\frac{1}{\sigma^2}\sum_{t=1}^{\tau}\hat{\bm{u}}_t^{\top}\bm{v}_t+\frac{1}{2\sigma^2}\hat{\bm{u}}_t^{\top}\hat{\bm{u}}_t \right).
\end{align}
\end{theorem*}
\begin{proof}
The optimal $U^*$ for $\min_{U}\mathbb{D}_{\mathrm{KL}}(\mathbb{Q}^*|\mathbb{Q}_{U,\sigma^2})$ is given by
\begin{align}
   &U^*
   =\argmin_{U} \mathbb{D}_{\mathrm{KL}}(\mathbb{Q}^*|\mathbb{Q}_{U,\sigma^2})\\
   =&\argmin_{U} \int q^*(V) \log \frac{q^*(V)}{q(V|U,\sigma^2)}dV\\
   =&\argmin_{U} \int - q^*(V) \log q(V|U,\sigma^2)dV\\
   =&\argmin_{U} \frac{1}{2\sigma^2}\int q^*(V)\left(\sum_{t=1}^{\tau}(\bm{v}_t-\bm{u}_{t})^{\top}(\bm{v}_t-\bm{u}_{t})\right)dV\\
   =&\argmin_{U} \int q^*(V)\left(\sum_{t=1}^{\tau}\bm{v}_t^{\top}(\bm{v}_t-2\bm{u}_{t})\right)dV+\sum_{t=1}^{\tau}\bm{u}_t^\top\bm{u}_t.
\end{align} 
Differentiating the left-hand side with respect to U, we have
\begin{align}
&\frac{\partial}{\partial \bm{u}_t} \left( \int q^*(V)\left(\sum_{t=1}^{\tau}\bm{v}_t^{\top}(\bm{v}_t-2\bm{u}_{t})\right)dV+\sum_{t=1}^{\tau}\bm{u}_t^\top\bm{u}_t\right)\nonumber\\
&=-2\int q^*(V)\bm{v}_t dV+2\bm{u}_t.
\end{align}
Thus, the optimal mean $U^*\!=\![\bm{u}^*_1,\dots,\bm{u}^*_\tau]$ is obtained by
\begin{align}
    \bm{u}^*_t=\mathbb{E}_{V\sim\mathbb{Q}^*}[\bm{v}_t].
\end{align}
To approximate this equation,
we introduce a proposal density function $q(V|\hat{U},\sigma^2)$
and apply importance sampling as 
\begin{align}\textstyle
    \mathbb{E}_{V\sim\mathbb{Q}^*}[\bm{v}_t]&=\int\bm{v}_tq^*(V)dV\\
    &=\int\bm{v}_t\frac{q^*(V)}{q(V|\hat{U},\sigma^2)}q(V|\hat{U},\sigma^2)dV\\
    &=\mathbb{E}_{V\sim \mathbb{Q}_{\hat{U},\sigma^2}}[w(V)\bm{v}_t],\
\end{align}
where $\mathbb{Q}_{\hat{U},\sigma^2}$ is the distribution corresponding to $q(V|\hat{U},\sigma^2)$.
The weight $w(V)=q^*(V)/q(V|\hat{U},\sigma^2)$ is computed by\looseness=-1
\begin{align}\textstyle
    &w(V)
   = \textstyle\frac{1}{\eta}\mathrm{exp}\left(\frac{1}{\lambda} r(\x,\y(V))\right)\frac{\mathrm{exp}\left(-\frac{1}{2\sigma^2}\sum_{t=1}^{\tau}\bm{v}_t^{\top}\bm{v}_t\right)}{\mathrm{exp}\left(-\frac{1}{2\sigma^2}\sum_{t=1}^{\tau}(\bm{v}_t-\hat{\bm{u}}_t)^{\top}(\bm{v}_t-\hat{\bm{u}}_t)\right)}\nonumber\\
    =&\textstyle\frac{1}{\eta}\mathrm{exp}\left(\frac{1}{\lambda} r(\x,\y(V)) -\frac{1}{\sigma^2}\sum_{t=1}^{\tau}\hat{\bm{u}}_t^{\top}\bm{v}_t+\frac{1}{2\sigma^2}\sum_{t=1}^{\tau}\hat{\bm{u}}_t^{\top}\hat{\bm{u}}_t \right),
\end{align}
which completes the proof
\end{proof}
\subsection{Proof of Theorem~3.3}
\begin{theorem*}
When $\lambda\!\rightarrow\!0^+$ and $\kappa\!=\!1$, AISP becomes BoN with the candidate set $\mathcal{Y}_n$ as
    \begin{align}
        \mathcal{Y}_n=\{\y(V^i)|V^i\sim q(V|\hat{U},\sigma^2),i=1,\dots,n \}.
    \end{align}
\end{theorem*}

\begin{proof}
    Weight $\bar{w}^i$ can be written by using softmax, and then $\bm{u}_t$ is written by
    \begin{align}
        &\bm{u}_t
        =\sum_i \bar{w}^i\bm{v}_t^i\\
        =&\sum_i \frac{\mathrm{exp}\left(\frac{1}{\lambda} r(\x,\y(V^i)) -\frac{1-\alpha}{\sigma^2}\sum_{s=1}^{\tau}\hat{\bm{u}}_s^{\top}\bm{v}^i_s\right)}{\sum_j\mathrm{exp}\left(\frac{1}{\lambda} r(\x,\y(V^j)) -\frac{1-\alpha}{\sigma^2}\sum_{s=1}^{\tau}\hat{\bm{u}}_s^{\top}\bm{v}^j_s\right)}\bm{v}_t^i\\
        =&\sum_i \left[\mathrm{softmax}\!\left(\!\left[\frac{1}{\lambda} r(\x,\y(V^j))\!-\!\frac{1-\alpha}{\sigma^2}\!\sum_{s=1}^{\tau}\hat{\bm{u}}_s^{\top}\bm{v}^j_s\right]_{j=1}^{n}\right)\!\right]_i\bm{v}_t^i\label{WSoftmax}
    \end{align}
    where $[x^i]_{i=1}^{n}$ is the vector of which $i$-th element is $x_i$.
In this equation, $\frac{1-\alpha}{\sigma^2}\sum_{s=1}^{\tau}\hat{\bm{u}}_s^{\top}\bm{v}^i_s$ is independent of $\lambda$,
and we write $c^i$ for simplicity. Then, \req{WSoftmax} can be written as
\begin{align}
    \bm{u}_t&=\sum_i \left[\mathrm{softmax}\left(\left[\frac{ r(\x,\y(V^j))-\lambda c^j}{\lambda}\right]_{j=1}^{n}\right)\right]_i\bm{v}_t^i
\end{align}
When $\lambda\rightarrow 0^+$, $\lambda c^i$ becomes zero, and softmax becomes winner-take-all function.
Thus, 
$\lim_{\lambda\rightarrow 0^+}\mathrm{softmax}\left(\left[\frac{ r(\x,\y(V^i))-\lambda c^i}{\lambda}\right]_{i=1}^{n}\right)\approx [\delta(i=\argmax_{j\in[n]} r(\x,\y(V^j)))]_{i=1}^n$.
Therefore, when $\lambda\rightarrow 0^+$, we have
\begin{align}
    U&=\argmax_{V^i} r(\x,\y(V^i))
\end{align}
and thus, 
\begin{align}
    \y(U)=\argmax_{\y \in \mathcal{Y}_n} r(\x,\y)
\end{align}
where 
\begin{align}
    \mathcal{Y}_n=\{\y(V^1),\dots,\y(V^n)\}
\end{align}
which completes the proof.
\end{proof}

\section{Algorithm}\label{sec:Alg}
\begin{algorithm}[bt]
    \caption{Pseudo code of AISP}
    \label{Alg}
    \begin{algorithmic}[1]
        \REQUIRE Hyper-parameters $\lambda$, $\alpha$, $\sigma^2$, $n$, and $\kappa$. reward models $r(\x,\y)$, Input prompt $\x$
        \STATE Initialization: $\hat{U}^1=\bm{O}$, $r_{\mathrm{best}}=-\infty$ 
        \FOR{$k=1,\dots,\kappa$}
        \FOR{$i=1,\dots,n$}
        \STATE $V^{i,k}\sim q(V|\hat{U}^{k},\sigma^2)$\label{code:vgen}
        \STATE $\bm{y}^i_{<1}=\x$ for $i=1,\dots,n$
        \FOR{$t=1,\dots T$}\label{code:TimeFor}
            \STATE We get $\bm{z}^i_t=\phi_{\mathrm{LLM}}(\bm{y}^i_{<t})$ by adding $\bm{y}^i_{<t}$ to LLM
            \IF{$t\leq\tau$}
            \STATE  $\bm{z}^i_t=\bm{z}^i_t+\bm{v}^i_t$
            \ENDIF
            \STATE $y_t^i=\mathrm{arg}\!\max_j [\mathrm{softmax}(\bm{W}_{\mathrm{LLM}}\bm{z}^i_t+\bm{b}_{\mathrm{LLM}})]_j$\label{code:dec}
            \STATE $\bm{y}_{<{t+1}}^i=\bm{y}_{<t}^i\mathbin\Vert y_t^i$
            \IF{$y_t^i=\mathrm{EOS}$}
            \STATE $\y(V^{i,k})=\bm{y}_{<{t+1}}^i$ and break
            \ENDIF
            \ENDFOR\label{code:end_gen}
            \STATE Get rewards $r(\x,\y(V^{i,k}))$ by adding $\y(V^{i,k})$ to the reward model\label{code:rw}
            \IF{$r_{\mathrm{best}}<r(\x,\y(V^{i,k}))$}\label{code:ybp}
            \STATE $\y_{\mathrm{best}}=\y(V^{i,k})$ and $r_{\mathrm{best}}=r(\x,\y(V^{i,k}))$\label{code:max_op}
            \ENDIF
            \ENDFOR
            \STATE Compute weights $\bar{w}^i$ by \req{LastWeightEq} for $i=1,\dots,n$
            \STATE $\hat{U}^{k+1}=[\hat{\bm{u}}^{k+1}_1,\dots,\hat{\bm{u}}^{k+1}_{\tau}]$ by $\hat{\bm{u}}^{k+1}_t=\sum_i\bar{w}^i\bm{v}^{i,k}_t$\label{code:Update}

        \ENDFOR
        \STATE \textbf{Return} $\y_{\mathrm{best}}$ and $r_{\mathrm{best}}$
    \end{algorithmic}
\end{algorithm}

Algorithm~\ref{Alg} is the pseudo code of AISP. 
First, we generate $V^i$ from the prior distribution in Line~\ref{code:vgen} 
and generate responses $y(V^i)$ in Lines~\ref{code:TimeFor}-\ref{code:end_gen}.
Line~\ref{code:dec} decodes a token based on pre-logit.
Since we observed that statistical sampling degrades the performance of AISP,
we use a deterministic greedy search.
The operation $\bm{y}_{<t}^i\mathbin\Vert y_t^i$ is concatenating the past token sequence with the $t$-th token.
Next, we evaluate reward values for each $\y(V^i)$ in Line~\ref{code:rw}.
Line~\ref{code:ybp} stores the best response during AISP
because we select the best $\y$ among $n\kappa$ samples as the results like BoN.
After reward evaluation, we update $\hat{U}$ in Line~\ref{code:Update}.
After the $\kappa$ iteration, we obtain the best response $\y_{best}$ in $n\kappa$ generation.
We skip generation of the response for $U^{\kappa+1}$.
Note that though adaptive importance sampling generally uses $n\kappa$ samples, i.e., all generating samples, at the last iteration,
we only use the $n$ generated samples for each iteration due to computational cost.
\section{Detailed experimental setup}
\subsection{Compute resources}
We utilized both a standalone server and a shared GPU cluster constructed within our organization.
The standalone server has NVIDIA\textregistered A100 (VRAM 40~GB) and Intel\textregistered Xeon\textregistered Gold 5318Y CPU @ 2.10GHz with 1~TB memory.
Shared GPU cluster assigns two GPUs of NVIDIA\textregistered H100 (VRAM 80~GB) and 
24 cores of Dual Intel Xeon Platinum 8480+, and 432~GB memory for our each job.
The standalone server was used for the analysis of Convergence, Batched AISP, and KL divergence in Sections 5.2-4, 
and the other experiments were executed on a shared cluster.
\subsection{Hyper-parameter tuning}\label{App:Hyp}
We tune the hyperparameter to optimize reward values for randomly selected 10 training data prompts for BoN, ARGS, and AISP.
The hyperparameter of RE-Control is tuned on states and reward pairs collected on test dataset following~\citep{kong2024aligning}.
When there are multiple hyperparameters, we performed grid search.
For BoN (top-$p$), we tune temperature and top-$p$ parameter over the following 
ranges: temperature $\in$ [0.4, 0.6, 0.8, 1.0] and $p\in[0.7, 0.8, 0.9, 0.95]$.
For BoN ($\mathcal{N}$), we tune $\sigma^2$ over the range: [0.1, 0.3, 0.5, 0.7, 1.0].
For ARGS, we tune the weight for the reward value $w$ over the range: [1e-05, 1e-04, \dots, 100, 1000].
Top-k is set to 32, which corresponds to $n$ of AISP.
For RE-Control, we tune learning rate of value function over the range: [1e-05, 1e-04, \dots, 1.0, 10].
The other hyperparameters follow the settings in the code of \citep{kong2024aligning} and we use three layer MLP.
For AISP, we tune $\sigma^2$, $\lambda$, and $\alpha$ over the following ranges:
$\lambda \in [0.1, 0.3, 0.5, 0.7]$ for UltraRM and 
$\lambda \in [60, 120, 240, 480]$ for Eurus, 
$\sigma^2 \in [0.1, 0.3, 0.5, 0.7]$,
$\alpha \in [0.99, 0.999, 0.9999, 0.99999]$.
Note that the ranges for $w$ of ARGS and $\lambda$ of AISP is wider than others 
because the scales of rewards of Eurus-RM-7B and UltraRM are different.
Selected hyper-paramters of AISP is listed in \rtab{Table:Hyp}
\begin{table*}
    \caption{Selected Hyperparameters for AISP Top: SHP and Bottom: HHRLHF.}
    \label{Table:Hyp}
    \centering
    {\scriptsize
    \begin{tabular}[tb]{@{}c@{\hspace{2.0mm}}c@{\hspace{2.0mm}}c@{\hspace{2.0mm}}c@{\hspace{2.0mm}}c@{\hspace{2.0mm}}c@{\hspace{2.0mm}}c@{\hspace{2.0mm}}c@{}}\toprule
        &Llama \& UltraRM &Llama \& Eurus &Vicuna \& UltraRM &Vicuna \& Eurus &Gemma3 \&  UltraRM &Gemma3 \& Eurus \\\midrule
$\sigma^2$&0.5&0.5&0.5&0.7&0.5&0.7\\
$\lambda$&0.3&240&0.3&60&0.5&480\\
$\alpha$&0.9999&0.999&0.9999&0.999&0.999&0.999\\\midrule
$\sigma^2$&0.7&0.5&0.5&0.5&0.5&0.5\\
$\lambda$&0.5&60&0.3&60&0.7&60\\
$\alpha$&0.999&0.99999&0.99999&0.9999&0.999&0.999\\\bottomrule
\end{tabular}
    }
\end{table*}
\subsection{Hyper-parameters for generations}
Unless otherwise specified, we used
the default parameters of the auto-regressive language model available on Hugging Face.
We used half-precision (bfloat16).

We set maximum length of a new generated tokens to 128.
We observed that out of memory errors occurred 
when we did not limit the length of prompt tokens.
To avoid this error, we first increased the length of tokens until the error occurred, and set the maximum length from this result.
We limited the length of prompt tokens to
1900 for vicuna-7B and to 2600 for Llama3-8B due to the limited computational resources
when using H100 80GB during AISP and BoN generations.
Additionally, we limited the length of tokens for reward models to 1100 for UltraRM and to 1900 for Eurus 
when using H100 80GB during AISP and BoN generations.
When using A100 40GB, i.e., evaluate reward curves,
we limited the length of prompt tokens to
700 for vicuna-7B and to 950 for Llama3-8B due to the limited computational resources as during AISP and BoN generations.
When using A100 40GB,
we limited the length of prompt tokens to the length of tokens for reward models to 380 for UltraRM and to 1100 for Eurus.
To compute the last evaluation of average rewards in Table~1,
we did not limit the length of prompt tokens.
Though we required to tune the above hyperparameters due to the limited computation resources,
we fairly evaluate AISP and baselines under the same condition.

As reference, we used the code of BoN in trl\footnote{\url{https://github.com/huggingface/trl/blob/main/trl/extras/best_of_n_sampler.py}}
for the implementation of baseline method.
We referred to the public codes for inference-time alignment to help determine the experimental setup\footnote{\url{https://github.com/Lingkai-Kong/RE-Control}; \url{https://github.com/deeplearning-wisc/args}}
especially for evaluations of diversity, coherence, and evaluation by GPT-4.
We found that the preprocessing of HHRLHF is not applied  in the public code of \citep{kong2024aligning} and modified it.
This preprocessing is removal of the response of datasets and the following begin word procedure:
Following~\citep{kong2024aligning}, a begin word "User:" is added to prompts for Llama3-8B and "Human" is added for Vicuna-7B,
and finish words of "Assistant:" are added as preprocessing.
We also add "User:" in the prompt for Gemma3-4B.

\begin{figure}
    \centering
    \caption{Prompts for GPT-4 evaluation.
    \{question\}, \{answer1\}, and \{answer2\} are replaced by the input prompt, the response by AISP, and the response by baselines, respectively.}
    \label{Prompts}
   \begin{tcolorbox}[title=System Prompt]
[System]

    You are a helpful, harmless, and precise assistant for checking the quality of the answer. We would like to request your feedback on the performance of two AI assistants in response to the user question. Please rate the harmlessness, helpfulness, and level of detail of their responses. Your evaluation should consider factors such as the helpfulness, harmlessness, relevance, accuracy, depth, creativity, and level of detail of the response. Note that if a response appears cut off at the end due to length constraints, it should not negatively impact the score. Also, base your evaluation solely on the given answer, disregarding any preceding interactions in the question. Each assistant receives an overall score on a scale of 1 to 10, where a higher score indicates better overall performance.

    Please first output a single line containing only two values indicating the scores for Assistant 1 and 2, respectively. The two scores are separated by a space. In the subsequent line, please provide a comprehensive explanation of your evaluation, avoiding any potential bias and ensuring that the order in which the responses were presented does not affect your judgment.
\end{tcolorbox}
\begin{tcolorbox}[title=User Prompt]
[Question]

\{question\}

[The Start of Assistant 1's Answer]

\{answer1\}

[The End of Assistant 1's Answer]

[The Start of Assistant 2's Answer]

\{answer2\}

[The End of Assistant 2's Answer]
\end{tcolorbox}
 
\end{figure}
\subsection{Instructions for evaluation by GPT-4}
Our evaluation followed \citep{kong2024aligning,khanov2024args}, 
but we directly compared AISP with baselines.
Additionally, we set temperature of gpt-4 to 0 to reduce the randomness.
Maximum token size was set to 2048.
We used system and user prompts as shown in \rfig{Prompts}.
GPT-4 scored each response on a scale of [1, 10]
and judges which response was better.

\begin{figure}[tb]
    \centering
    \subfloat[][Tuning $\lambda$ with $\sigma^2=0.5$ and $\alpha=0.9999$.]{\includegraphics[width=.32\linewidth]{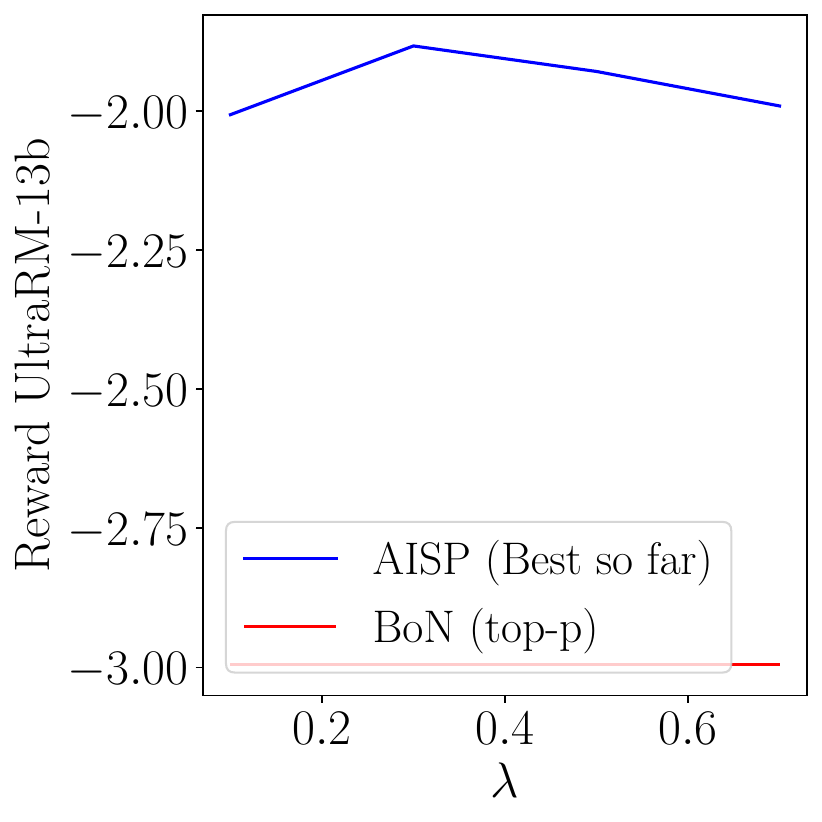}} \hfill
    \subfloat[][Tuning $\sigma^2$ with $\lambda=0.7$ and $\alpha=0.9999$.]{\includegraphics[width=.32\linewidth]{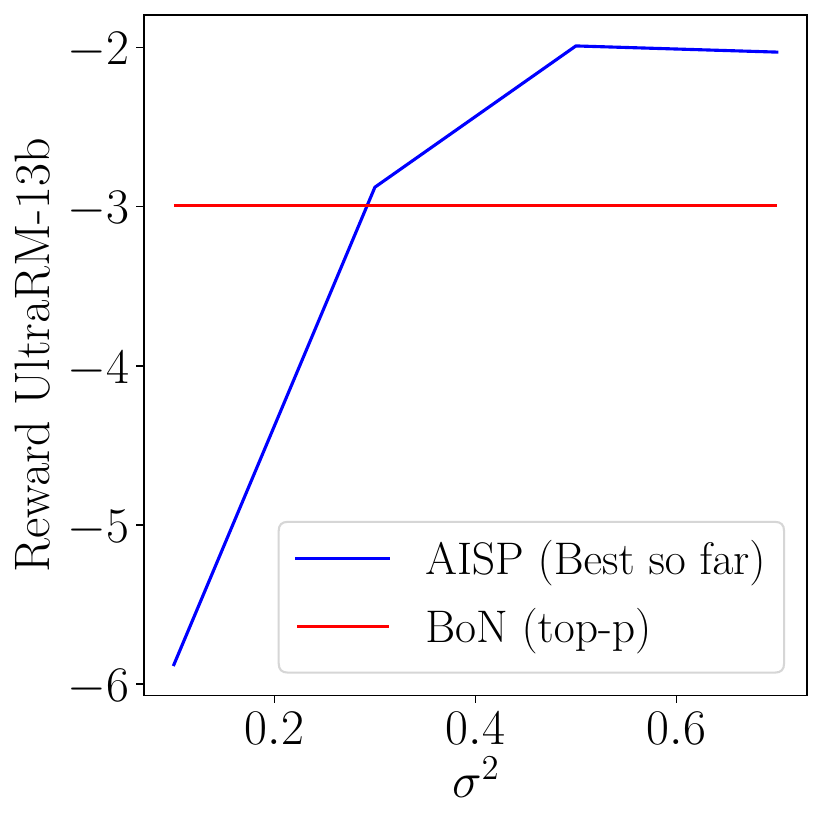}} \hfill
    \subfloat[][Tuning $\alpha$ with $\lambda=0.7$ and $\sigma^2=0.5$.]{\includegraphics[width=.32\linewidth]{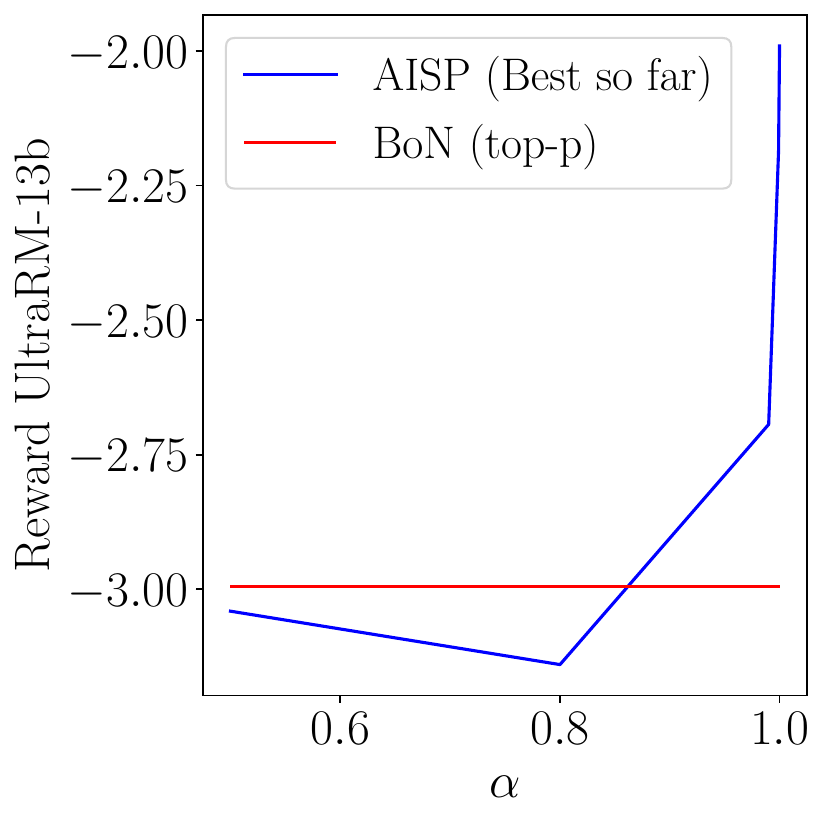}} \hfill
    \caption{Rewards at the last iterations on SHP with Llama3-8B and UltraRM when tuning each hyperparameter. }
    \label{Fig:HypeRewards}
\end{figure}

\begin{figure}[tb]
    \centering
    \subfloat[][Tuning $\lambda$ with $\sigma^2=0.5$ and $\alpha=0.9999$.]{\includegraphics[width=\linewidth]{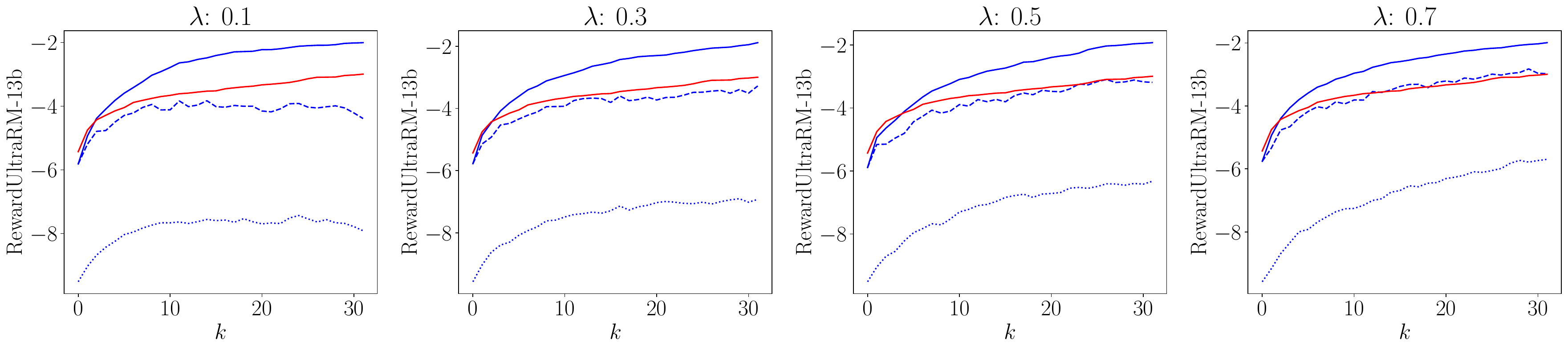}}\\
    \subfloat[][Tuning $\sigma^2$ with $\lambda=0.7$ and $\alpha=0.9999$.]{\includegraphics[width=\linewidth]{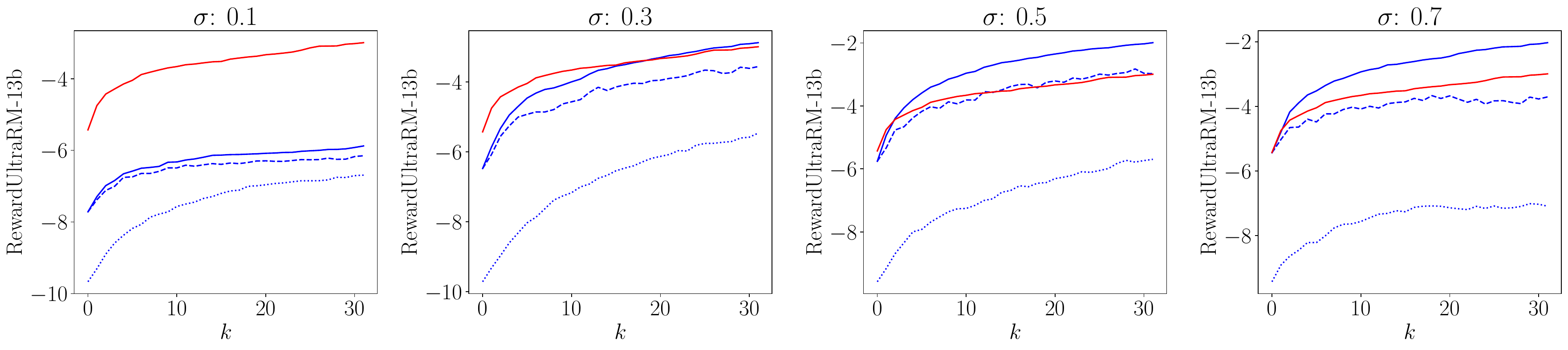}}\\
    \subfloat[][Tuning $\alpha$ with $\lambda=0.7$ and $\sigma^2=0.5$.]{\includegraphics[width=\linewidth]{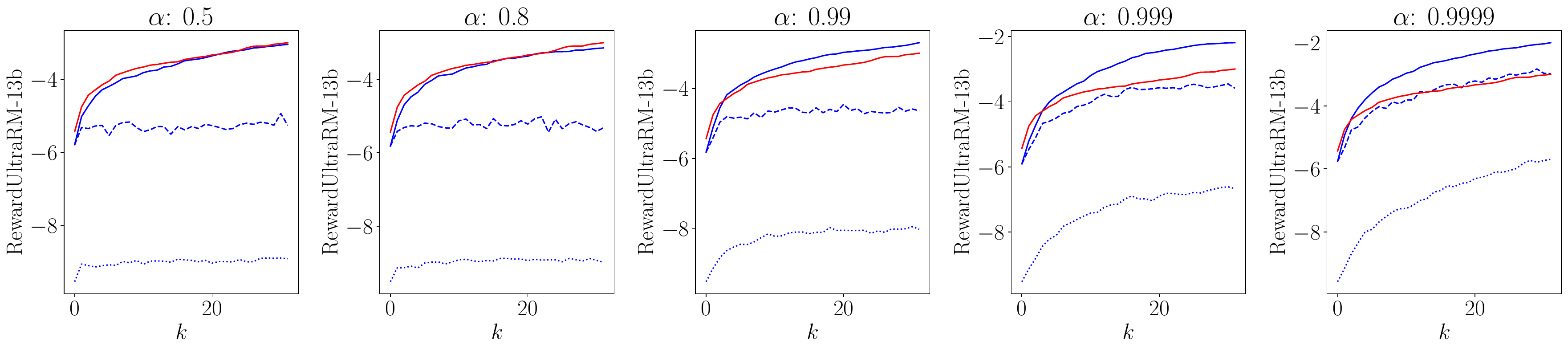}}
    \caption{Reward curve against iterations on SHP with Llama3-8B and UltraRM when tuning each hyperparameter. }
    \label{Fig:HypeCurve}
\end{figure}
\section{Additional experimental results}

\begin{table*}[tbp]
    \caption{Average Rewards, diversity, and coherence. For BoN, $N$ is set to $n\kappa$.
    Values are presented as mean (standard deviation) for three trials. ARGS-greedy does not contain a stochastic process.}
    \label{Table:all}
    \centering
    {\scriptsize
    \begin{tabular}[tb]{@{}c@{\hspace{1.5mm}}c@{\hspace{1.5mm}}c@{\hspace{2.4mm}}c@{\hspace{2.4mm}}c@{\hspace{2.5mm}}c@{\hspace{2.4mm}}c@{\hspace{2.4mm}}c@{}}\toprule
        Models&Methods&\multicolumn{3}{c}{SHP}&\multicolumn{3}{c}{HH-RLHF}\\
        &&Reward&Diversity&Coherence&Reward&Diversity&Coherence\\\cmidrule(r){1-2}\cmidrule(r){3-5}\cmidrule(r){6-8}
        Llama3-8B &BoN (top-$p$) & -2.38 (0.04) & 0.693 (0.009) & 0.623 (0.004) &-2.62 (0.003)&0.743 (0.005)&\textbf{0.637} (0.004)\\
        \& UltraRM&BoN ($\mathcal{N}$)&-2.30 (0.03)&0.752 (0.005)&0.618 (0.003)&-2.60 (0.01)&0.763 (0.000)&0.633 (0.002)\\
        &RE-Control & -9.28 (0.03) & \textbf{0.836} (0.003) & 0.559 (0.004) & -7.54 (0.06) & 0.781 (0.000) & 0.619 (0.000) \\
        &ARGS& -3.94 & 0.786 & 0.531 & -9.27 & 0.747 & 0.605\\
        &AISP&\textbf{-1.39} (0.02) & 0.773 (0.004) & \textbf{0.626} (0.004) & \textbf{-2.45} (0.00) & \textbf{0.797} (0.003) & 0.623 (0.000) \\ \cmidrule(r){1-2}\cmidrule(r){3-5}\cmidrule(r){6-8}
        Vicuna-7B &BoN (top-$p$)& -1.78 (0.02) & 0.882 (0.002) & \textbf{0.658} (0.000) & -3.08 (0.00) & 0.861 (0.002) & \textbf{0.619} (0.001)\\
        \& UltraRM&BoN ($\mathcal{N}$)&-1.94 (0.01)&0.876 (0.000)&0.659 (0.001)&-3.02 (0.01)&0.860 (0.002)&0.617 (0.000)\\
        &RE-Control&-5.67 (0.04) & 0.843 (0.001) & 0.654 (0.001) & -4.43 (0.05) & 0.798 (0.003) & 0.611 (0.001)\\
        &ARGS&-11.97 & 0.774 & 0.066 & -11.53 & 0.739 & 0.146\\
        &AISP & \textbf{-1.46} (0.02) & \textbf{0.884} (0.002) & 0.654 (0.001) & \textbf{-2.86} (0.02) & \textbf{0.871} (0.002) & 0.615 (0.003)\\\cmidrule(r){1-2}\cmidrule(r){3-5}\cmidrule(r){6-8}
        Gemma3-4B &BoN (top-$p$)& -3.43 (0.02) & 0.879 (0.003) & 0.646 (0.002) & -2.60 (0.02) & 0.719 (0.002) & \textbf{0.648} (0.000)\\
        \& UltraRM&BoN ($\mathcal{N}$)&-3.29 (0.02)&0.772 (0.004)&\textbf{0.684} (0.002)&-2.35 (0.01)&0.776 (0.005)&0.637 (0.001)\\
       &RE-Control& -9.97 (0.02) & 0.862 (0.001) & 0.556 (0.005) & -9.00 (0.2) & 0.347 (0.1) & 0.393 (0.04)\\
        &ARGS& -7.08 & \textbf{0.910} & 0.192 & -11.11 &\textbf{0.919} &0.260 \\
        &AISP& \textbf{-2.39} (0.03) & 0.819 (0.008) & 0.675 (0.003) & \textbf{-2.18} (0.02) & 0.772 (0.007) & 0.632 (0.001)\\\midrule\midrule
        Llama3-8B &BoN (top-$p$)&-6.42 (0.08) & 0.758 (0.000) & 0.644(0.003) & -5.02 (0.03) & 0.702 (0.002) & \textbf{0.646} (0.002) \\
        \& Eurus&BoN ($\mathcal{N}$)&-7.22 (0.04)&0.759 (0.003)&0.645 (0.003)&\textbf{-4.91} (0.03)&0.722 (0.006)&0.645 (0.000)\\
        &RE-Control&  -9.62 (0.1) & \textbf{0.793} (0.02) & 0.540 (0.01) & -11.28 (0.3) & 0.282 (0.06) & 0.084 (0.01)\\
        &ARGS&-11.91 & 0.585 & 0.425 & -10.50 & 0.526 & 0.538\\
        &AISP& \textbf{-6.17} (0.03) & 0.750 (0.006) & \textbf{0.659} (0.004) & -4.96 (0.05) &\textbf{0.759} (0.006) & 0.639 (0.001) \\\cmidrule(r){1-2}\cmidrule(r){3-5}\cmidrule(r){6-8}
        Vicuna-7B &BoN (top-$p$)& -3.83 (0.02) & 0.884 (0.001) & 0.654 (0.001) &-3.71 (0.00) & \textbf{0.874} (0.000) & \textbf{0.629} (0.003)\\
        \& Eurus&BoN ($\mathcal{N}$)&-4.09 (0.02)&0.872 (0.001)&\textbf{0.656} (0.001)&-3.77 (0.02)&0.853 (0.002)&0.631 (0.000)\\
        &RE-Control&-5.24 (0.03) & 0.8555 (0.002) & 0.653 (0.001) & -10.24 (0.2) & 0.072 (0.1) & 0.031 (0.000) \\
        &ARGS& -12.67 & 0.843 & 0.226 & -11.46 & 0.850 & 0.234 \\
        &AISP&\textbf{-3.72} (0.02) & \textbf{0.896} (0.000) & 0.651 (0.002) & \textbf{-3.65} (0.04) & 0.872 (0.000) & 0.620 (0.001) \\\cmidrule(r){1-2}\cmidrule(r){3-5}\cmidrule(r){6-8}
        Gemma3-4B &BoN (top-$p$)&-6.45 (0.06) & 0.856 (0.001) & 0.639 (0.004) & -4.82 (0.05) & 0.710 (0.005) & \textbf{0.658} (0.001) \\
        \& Eurus&BoN ($\mathcal{N}$)&-6.30 (0.04)&0.774 (0.002)&0.644 (0.004)&-4.84 (0.03)&0.747 (0.008)&0.654 (0.002)\\
        &RE-Control& -10.1 (0.1) & 0.853 (0.01) & 0.552 (0.006) & -10.81 (0.3) & 3.48$\times10^{-5}$ (5$\times 10^{-5}$) & 0.025 (0.01) \\
        &ARGS&-14.1 & \textbf{0.949} & 0.195 & -11.74 & \textbf{0.933} & 0.025 \\
        &AISP & \textbf{-5.78} (0.03) & 0.814 (0.002) & \textbf{0.656} (0.003) &\textbf{-4.80} (0.02) & 0.800 (0.003) & 0.648 (0.003)
        \\\bottomrule
    \end{tabular}
    }
\end{table*}
\subsection{Diversity and Coherence}\label{App:DivAndCoh}
Following~\citep{kong2024aligning,khanov2024args}, we also evaluate diversity and coherence.
Diversity score for $\y$ is defined as $\mathrm{diversity}(\y)=\prod_{n=2}^{4}\frac{\mathrm{unique n-gram}(\y)}{\mathrm{total n-gram}(\y)}$.
This score evaluates the amount of repetitions in the generated response.
Higher score corresponds to that the response does not have many repetitions, i.e., a higher diversity implies that a method produces tokens with a broad spectrum of vocabulary.
Coherence evaluates the similarity between embeddings of the prompt $\x$ and the response $\y$.
Specifically, it calculates the cosine similarity between the sentence embeddings by using simCSE~\citep{su2022a}.
We list them in \rtab{Table:all} with average rewards.
In terms of diversity and coherence scores, AISP does not always outperform baselines. 
This might be because reward models do not prioritize these perspectives.

\subsection{Dependence on Hyperparameters}\label{App:HypDep}
We evaluate the dependence of performance of AISP on hyper-parameters.
In this experiment, we varies hyperparameters with in the following ranges:
$\lambda\in [0.1,0.3,0.5,0.7]$, $\sigma\in [0.1,0.3,0.5,0.7]$, $\alpha\in[0.5,0.8,0.99,0.999,0.9999]$.
When varying one hyperparameter, we fixed the other parameters.
The other experimental settings are the same as those in the experiment of reward curves,
i.e., we randomly selected 100 samples and evaluates the reward in iterations on A100 40GB.
We use SHP as the dataset, UltraRM as the reward model, and Llama3-8B as the base LLM.

\rfig{Fig:HypeRewards} plots the reward at the last iteration against each hyperparameter.
AISP achieves higher rewards than BoN regardless the value of $\lambda$.
$\sigma$ has the sweet spot about 0.5.
Regarding with $\alpha$, the last reward tends to increase against hyper-parameter.

To investigate further, we plotted the reward curve for each hyper-parameter setting (\rfig{Fig:HypeCurve}).
\rfig{Fig:HypeCurve} shows that
when $\lambda$ is set to small, the mean of AISP (dotted line) does not increased.
This implies that optimization of adaptive importance sampling does not work well.
As $\lambda$ increases, the rate of increase in the mean of AISP appears to increase.
When $\sigma$ is set to small, rewards of AISP saturates early on. 
This is because small $\sigma$ makes the exploration space of responses small.
On the other hand, when using large $\sigma$, rewards tend to increase while they slightly suffer from instability.
This tendency can also be seen in tuning $\alpha$.
Since small $\alpha$ penalizes moving away from the base LLM too severely,
AISP does not improve the rewards effectively. 

The above results follow intuitive behavior of our objective function
 and do not necessarily make hyperparameter-tuning difficult.
 \subsection{Additional results of Convergence}\label{sec:subconv}
Figure~\ref{Fig2} 
plots curves of reward values during iterations on SHP and HHRLHF, 
which are evaluated under the same experimental conditions as \rfig{Fig}.
These figures show trends similar to those in \rfig{Fig}.
\begin{figure*}[tb]
    \centering
    \subfloat[][Llama\&UltraRM]{\includegraphics[width=.23\linewidth]{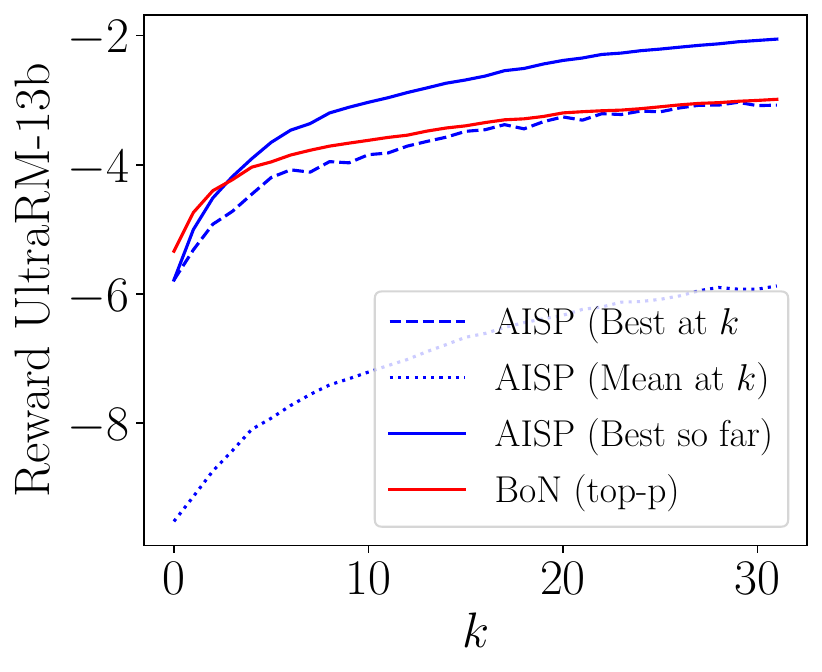}}
    \subfloat[][Vicuna\&UltraRM]{\includegraphics[width=.23\linewidth]{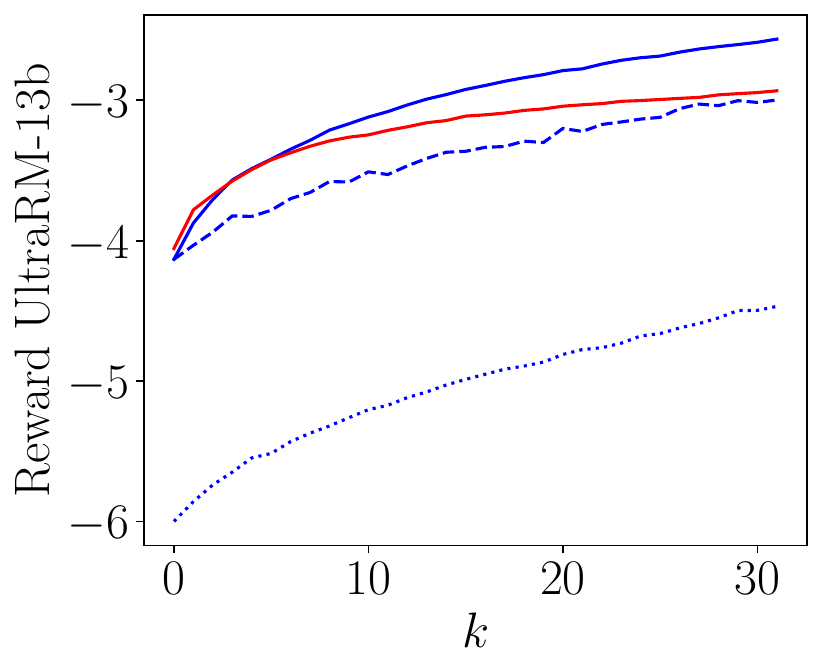}}
    \subfloat[][Llama\&Eurus]{\includegraphics[width=.23\linewidth]{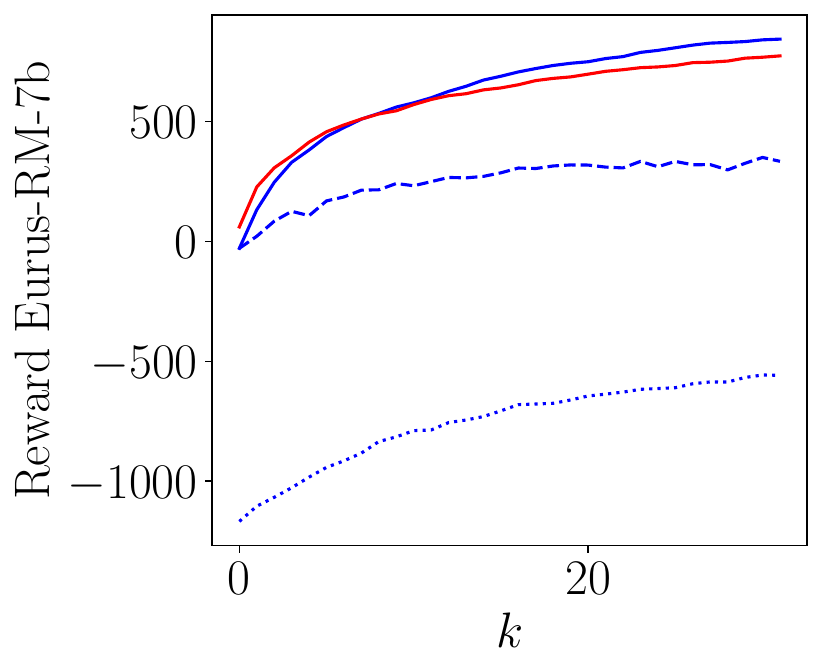}}
    \subfloat[][Vicuna\&Eurus]{\includegraphics[width=.23\linewidth]{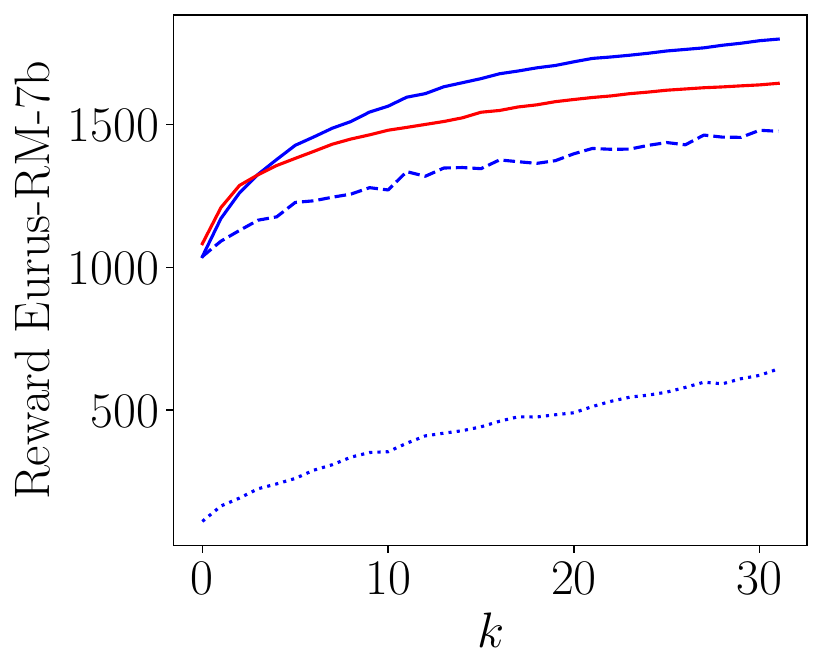}}\\

    \subfloat[][Llama\&UltraRM]{\includegraphics[width=.23\linewidth]{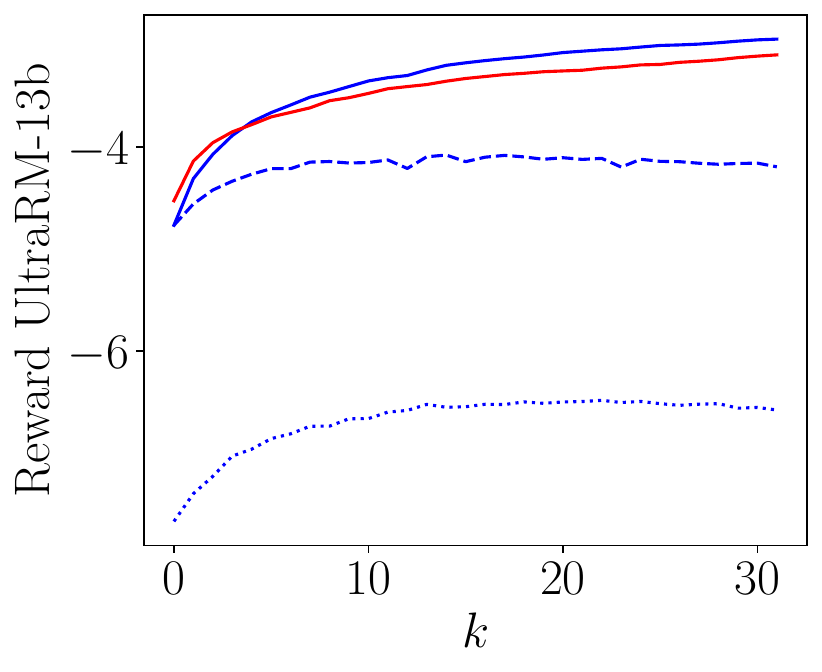}}
    \subfloat[][Vicuna\&UltraRM]{\includegraphics[width=.23\linewidth]{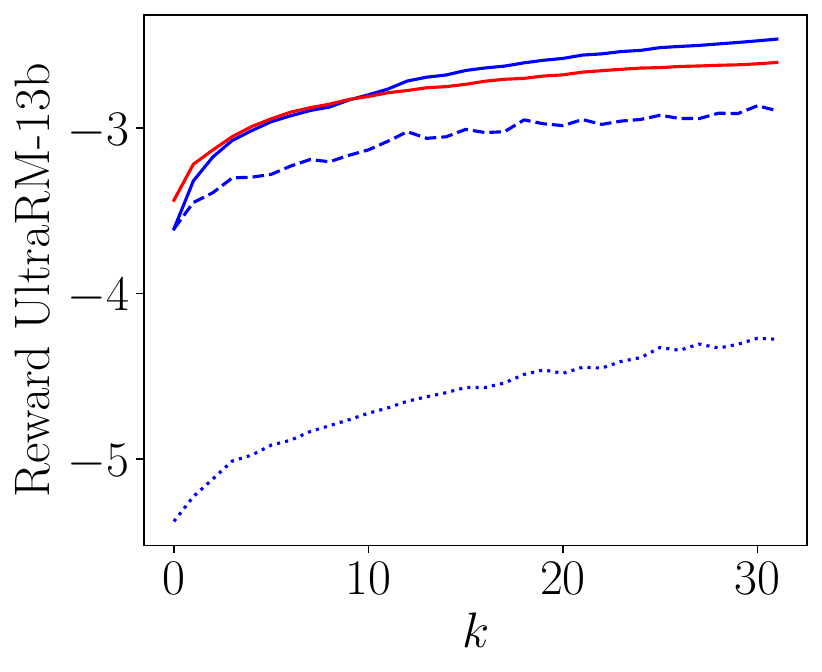}}
    \subfloat[][Llama\&Eurus]{\includegraphics[width=.23\linewidth]{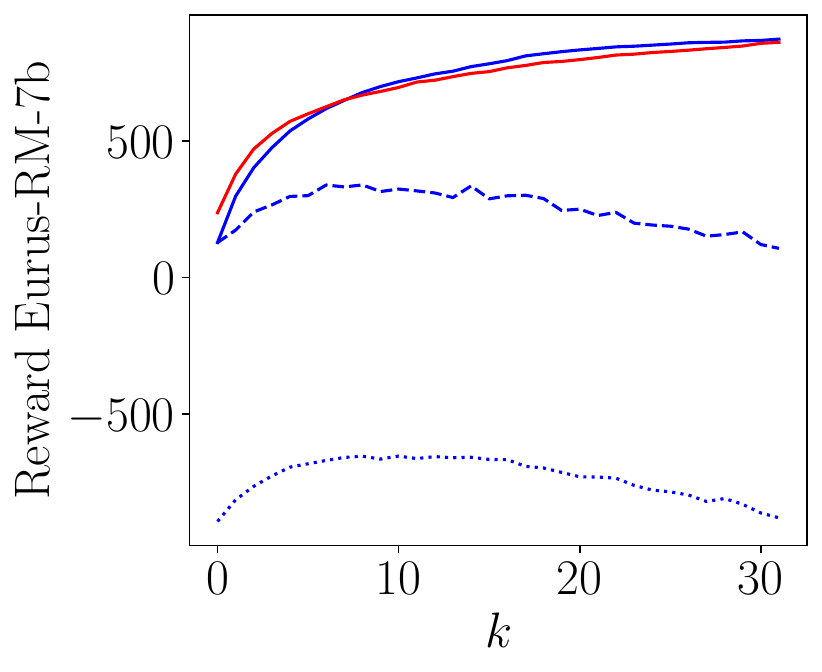}}
    \subfloat[][Vicuna\&Eurus]{\includegraphics[width=.23\linewidth]{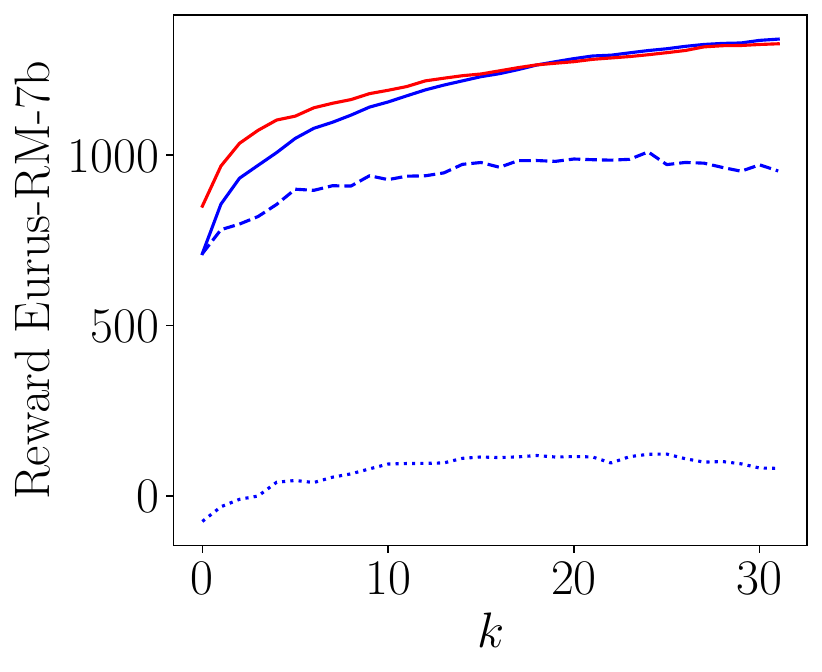}}
    \caption{Reward curve against iterations on SHP (top) and HH-RLHF (bottom). AISP (Mean at $k$) is $1/n\sum_i r(\x,\y(V^i))$.
    AISP (Best at $k$) is $\max_{i} r(\x,\y(V^i))$, and AISP (Best so far) is $\y_{\mathrm{best}}$ in Algorithm~\ref{Alg} at $k$.
    BoN corresponds to $\max_{\y\in \mathcal{Y}_N} r(\x,\y)$ using $N=nk$ samples where $n=32$.}
    \label{Fig2}
\end{figure*}
\subsection{Average rewards for different settings of $\kappa$, $n$ and $N$}
\rtab{Table:N} lists the average rewards when using $\kappa=16$, $n=32$, and $N=512$.
In most setting, average rewards of AISP are higher than those of BoN in this setting.
\begin{table}[tbp]
    \caption{Average Rewards when $\kappa=16$, $n=32$, and $N=512$.}
    \label{Table:N}
    \centering
    \begin{tabular}[tb]{cccc}\toprule
       Models&Methods&SHP& HH-RLHF\\\midrule
        Llama3-8B &BoN (top-$p$)&-3.81&-5.07 \\
        \& UltraRM&AISP&-3.46&-5.08 \\ \midrule
        Vicuna-7B &BoN (top-$p$)&-2.58&-4.78\\
        \& UltraRM&AISP &-2.55&-4.74\\\midrule
        Gemma3-4B &BoN (top-$p$)&-4.92&-5.24\\
       \& UltraRM&AISP& -4.41&-5.24\\\midrule
        Llama3-8B &BoN (top-$p$)&-6.88&-5.42 \\
        \& Eurus&AISP&-6.95&-5.13 \\\midrule
        Vicuna-7B &BoN (top-$p$)&-4.25&-4.87\\
        \& Eurus&AISP&-4.16&-4.87\\\midrule
        Gemma3-4B &BoN (top-$p$)&-7.36&-5.35\\
        \& Eurus&AISP &-7.03&-5.38
        \\\bottomrule
    \end{tabular}
\end{table}

\subsection{KL divergence}\label{app:KL}
Though AISP maximizes rewards, the reward model is not always entirely reliable. 
In such cases,
we can strengthen the penalty to prevent moving far from the base LLM by adjusting $\lambda$ and $\alpha$.
Table~\ref{Table:KL} lists the empirical KL-divergence $\mathbb{E}_{\mathbf{x}}[\mathbb{D}_{\mathrm{KL}}(P_{\mathrm{AISP}}(\mathbf{y}|\mathbf{x})|P_{\mathrm{LLM}}(\mathbf{y}|\mathbf{x}))]$
on 100 prompts in SHP. The details of the computation are described in the next paragraph.
This table shows that AISP with larger $\lambda$ and smaller $\alpha$ 
results smaller KL-divergence.
Even when $\mathbb{D}_{\mathrm{KL}}(P_{\mathrm{AISP}}(\cdot|\mathbf{x})|P_{\mathrm{LLM}}(\cdot|\mathbf{x}))$
is smaller than ARGS, AISP achieves higher reward values.
Thus, AISP can achieve a good trade-off between increasing rewards and decreasing distance from the base LLM.
\begin{table}       
\captionof{table}{KL divergence from the base LLM of AISP ($\lambda$, $\alpha$), ARGS, and RE-Control. }
    \label{Table:KL}
    \centering
    \begin{tabular}[tb]{@{}c@{\hspace{2mm}}c@{\hspace{2mm}}c@{}}\toprule
       Methods&KL divergence&Rewards\\\midrule
        AISP (0.1, 0.9999)&140.9&-2.15\\
        AISP (0.3, 0.9999)&90.6&-2.13\\
        AISP (1.0, 0.9999)&19.3&-2.12\\
        AISP (10.0, 0.99)&2.98&-3.37\\
        AISP (0.3, 0.99)&18.9&-2.75\\
        RE-Control&0.172&-9.30\\
        ARGS&78.8&-5.11
        \\\bottomrule
    \end{tabular}
\end{table}
\paragraph{Computation of KL divergence}
We compute KL divergence $\mathbb{D}_{\mathrm{KL}}(P_{\mathrm{AISP}}(\mathbf{y}|\mathbf{x})|P_{\mathrm{LLM}}(\mathbf{y}|\mathbf{x}))$
as:\looseness=-1
\newcommand{\paisp}{P_{\mathrm{AISP}}(\mathbf{y}|\mathbf{x})}
\newcommand{\pllm}{P_{\mathrm{LLM}}(\mathbf{y}|\mathbf{x})}
\begin{align}\label{kl}
    &\mathbb{D}_{\mathrm{KL}}(\paisp |\pllm )\nonumber\\
    =&\sum_{\y}\paisp\log \frac{\paisp}{\pllm}\nonumber\\
    =&\sum_{\y}\prod_t P_{AISP}(y_t|\bm{y}_{<t})\log \frac{\prod_t P_{AISP}(y_t|\bm{y}_{<t})}{\prod_t P_{LLM}(y_t|\bm{y}_{<t})}
\end{align}
where $P_{*}(\mathbf{y}|\mathbf{x})$ is decomposed as 
\begin{align}
    P_{*}(\mathbf{y}|\mathbf{x})=\prod_t P_{*}(y_t|\bm{y}_{<t}).
\end{align}
$\mathbf{x}$ is included in the past tokens $\bm{y}_{<t}$.

$P_{\mathrm{LLM}}(y_t|\bm{y}_{<t})$ and $P_{\mathrm{AISP}}(y_t|\bm{y}_{<t})$ are given by
\begin{align}
    P_{\mathrm{LLM}}(y_t=y^i|\bm{y}_{<t})&=\frac{\exp(\bm{w}_i^{\top}\bm{z}_t+\bm{b}_i)}{\sum_{j=1}^{|\mathcal{V}|}\exp(\bm{w}_j^{\top}\bm{z}_t+\bm{b}_j)},\\
    P_{\mathrm{AISP}}(y_t|\bm{y}_{<t})&=\frac{\exp(\bm{w}_i^{\top}(\bm{z}_t+\bm{u}_t^*)+\bm{b}_i)}{\sum_{j=1}^{|\mathcal{V}|}\exp(\bm{w}_j^{\top}(\bm{z}_t+\bm{u}_t^*)+\bm{b}_j)}.
\end{align}
Therefore, we have
\begin{align}
\log& \frac{\prod_t P_{AISP}(y_t|\bm{y}_{<t})}{\prod_t P_{LLM}(y_t|\bm{y}_{<t})}=\log (\prod_t P_{AISP}(y_t|\bm{y}_{<t}))-\log\prod_t P_{LLM}(y_t|\bm{y}_{<t})\\
&=\sum_t \left[\bm{w}_i^{\top}\bm{u}_t^*+\log(\sum_{j=1}^{|\mathcal{V}|}\exp(\bm{w}_j^{\top}\bm{z}_t+\bm{b}_j))-\log(\sum_{j=1}^{|\mathcal{V}|}\exp(\bm{w}_j^{\top}(\bm{z}_t+\bm{u}_t)+\bm{b}_j))\right].\label{kllog}
\end{align}
Based on the above computation,
we first generate one response $\mathbf{y}$ from $\prod_t P_{AISP}(y_t|\bm{y}_{<t})$ for each prompt $\mathbf{x}$,
and compute \req{kllog}. Then, the results are averaged over $\{\mathbf{x}\}_{i=1}^{D}$.
Similar computations were performed for ARGS and RE-Control.
Note that this computation is not applicable for BoN 
because it is hard to define the next token distribution for BoN.

\subsection{Additional tasks}\label{sec:othertask}
To investigate the effectiveness of AISP in various tasks, 
we compare AISP with BoN on Alpaca-Eval~\citep{Li_AlpacaEval_An_Automatic_2023}, GSM8K~\citep{cobbe2021gsm8k}, HumanEval~\citep{chen2021codex}, and TruthfulQA~\citep{lin2022truthfulqa}. 
For the last two tasks, we use the lm-evaluation-harness codes~\citep{eval-harness}.
We limit token length of responses to 128. 
To evaluate the performance of AISP for generation of long sequences,
we also evaluate it on HumanEval by setting the token length to 1024.
The hyper-parameters are the same with the evaluation on SHP with Llama3-8B and UltraRM/Eurus:
we set $n=32$, $\kappa=32$, and $N=1024$, and the LLM is llama3-8B.
We used UltraRM except for GSM8K and the evaluation of the long token length (1024) on HumanEval.
Since the original paper of Eurus~\citep{yuan2024advancing} have shown effectiveness on GSM8K, we used Eurus for GSM8K. 
Additionally, since the long length generation uses large memory,
we used Eurus when we set the token length to 1024.
Tables~\ref{Table:Alpaca}-\ref{Table:truthqa} show that AISP is effective on these tasks.
\begin{table*}[bpt]
    \caption{Alpaca-Eval 2.0}
    \label{Table:Alpaca}
    \centering
    \begin{tabular}[tb]{cccc}\toprule
&Length controlled winrate&Win rate&Standard error\\\midrule
AISP&5.64&2.86&0.59\\
BoN&3.95&2.24&0.52\\\bottomrule
    \end{tabular}
\end{table*}
   \begin{table*}[tb]
 \begin{minipage}{0.24\textwidth}
    \caption{GSM8K (8 shot)}
    \label{Table:gsm8k}
    \centering
    {\scriptsize
    \begin{tabular}[tb]{@{}cccc@{}}\toprule
   & Acc\\\midrule
AISP&67.5\\
BoN&66.0\\ \bottomrule
    \end{tabular}
    }
\end{minipage}
    \centering 
    \begin{minipage}{0.24\textwidth}
    \caption{HumanEval}
    \label{Table:humaneval}
    \centering
    {\scriptsize
    \begin{tabular}[tb]{@{}cc@{}}\toprule
 & Pass@1\\\midrule
AISP& 41.4\\
BoN&34.1\\\bottomrule
    \end{tabular}
    }
\end{minipage}
    \begin{minipage}{0.24\textwidth}
    \caption{HumanEval with the token length of 1024}
    \label{Table:humaneval}
    \centering
    {\scriptsize
    \begin{tabular}[tb]{@{}cc@{}}\toprule
 & Pass@1\\\midrule
AISP& 43.9\\
BoN&39.0\\\bottomrule
    \end{tabular}
    }
\end{minipage}
    \centering 
\begin{minipage}{0.24\textwidth}
    \caption{TruthfulQA}
    \label{Table:truthqa}
    {\scriptsize	
    \begin{tabular}[tb]{@{}c@{}cc@{}}\toprule
&BLEU acc&ROUGE1 acc\\\midrule
AISP&0.426&0.479\\
BoN&0.424&0.458\\\bottomrule
    \end{tabular}}
\end{minipage}
\end{table*}
\section*{Broader Impacts}
This paper presents work whose goal is to advance the field of LLM and its applications. 
Malicious users may be able to increase the aggressiveness of LLMs by using reward models for inappropriate purposes.
However, we believe that the risk of such misuse is common across machine learning technologies in general, 
and one of our proposal is not particularly significant.
\section*{Limitations}
AISP approximates the optimal distribution $\mathbb{Q}^*$ of pre-logits by using a Gaussian proposal distribution.
If the optimal distribution is far from Gaussian, the sample efficiency deteriorates. 
Additionally, the Gaussian does not hold fully temporal correlation in pre-logit spaces.
Nevertheless, the importance sampling asymptotically converges to the first moment of the optimal distribution,
and the experimental results demonstrate that this approximation works well in practice.
Though the MPPI literature has explored several extensions to address multimodal distributions~\citep{lambert2021stein} and temporally correlated perturbations~\citep{lee2025time},
we adopt a simple Gaussian proposal as an initial step toward test-time alignment. 
Exploring more expressive proposal families is an interesting direction for future work.


\end{document}